\documentclass[a4paper,12pt]{article}
\pdfoutput=1

\usepackage{graphicx} 
\usepackage[T1]{fontenc} 
\usepackage{parskip} 
\usepackage{latexsym,amsmath,amssymb} 
\usepackage[square,sort&compress,comma,numbers]{natbib} 
\usepackage{times} 
\usepackage{inconsolata} 
\usepackage[a4paper,left=3cm,right=3cm,top=3cm,bottom=3cm]{geometry} 

\usepackage[UKenglish]{babel}
\usepackage{xr} 
\usepackage{enumitem} 
\usepackage{verbatim} 
\usepackage{mathtools, amsthm} 
\usepackage{bbm} 
\usepackage{booktabs} 
\usepackage{caption} 
\usepackage{hyperref} 
\usepackage[capitalise]{cleveref} 
\usepackage[dvipsnames]{xcolor, colortbl} 
\usepackage{subcaption} 
\usepackage{doi} 
\usepackage{lscape} 
\usepackage{dirtytalk} 

\hypersetup{ 
    colorlinks=true,
    linkcolor=Blue,
    linktocpage=false,
    citecolor=Green,
    bookmarks=false,
    pdftitle={Distribution-Free Finite-Sample Guarantees and Split Conformal Prediction},
    pdfauthor={Roel Hulsman}
    }
\numberwithin{equation}{section}

\newtheorem{proposition}{Proposition}[]

\theoremstyle{definition}
\newtheorem{definition}{Definition}[]
\theoremstyle{remark}
\newtheorem{remark}{Remark}[section]

\DeclareRobustCommand{\abbrevcrefs}{%
    \crefname{theorem}{Theorem}{Theorems.}%
    \crefname{proposition}{Proposition}{Propositions}%
    \crefname{corollary}{Corollary}{Corollaries}%
    \crefname{lemma}{Lemma}{Lemmas}%
    \crefname{definition}{Definition}{Definitions}%
    \crefname{remark}{Remark}{Remarks}%
}

\DeclareRobustCommand{\cshref}[1]{{\abbrevcrefs\cref{#1}}}

\definecolor{lavender}{rgb}{0.9, 0.9, 0.98}

\begin{document}

\pagenumbering{gobble} 
\begin{titlepage}
\begin{center}
\vspace{1cm}

\textsf{\Huge{University of Oxford \\}}

\vspace{1cm}

\begin{figure}[htb]
\centering
\includegraphics[scale=1.5]{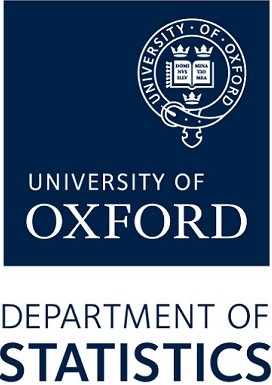}
\end{figure}

\vspace{2.0cm}

\Huge{Distribution-Free Finite-Sample Guarantees and Split Conformal Prediction}\\

\vspace{1.5cm}

\large{ by \\[14pt] Roel Hulsman \\[8pt] St. Anne's College} \\

\vspace{2cm}

\large{A dissertation submitted in partial fulfilment of the degree of Master of Science in Statistical Science.
} \\

\vspace{1cm}

\large{\emph{Department of Statistics, 24--29 St Giles, \\Oxford, OX1 3LB}} \\

\vspace{1.0cm}

\large{September 2022} \\

\end{center}
\end{titlepage}

\clearpage

\begin{abstract}

Modern black-box predictive models are often accompanied by weak performance guarantees that only hold asymptotically in the size of the dataset or require strong parametric assumptions. In response to this, split conformal prediction represents a promising avenue to obtain finite-sample guarantees under minimal distribution-free assumptions. Although prediction set validity most often concerns marginal coverage, we explore the related but different guarantee of tolerance regions, reformulating known results in the language of nested prediction sets and extending on the duality between marginal coverage and tolerance regions. Furthermore, we highlight the connection between split conformal prediction and classical tolerance predictors developed in the 1940s, as well as recent developments in distribution-free risk control. One result that transfers from classical tolerance predictors is that the coverage of a prediction set based on order statistics, conditional on the calibration set, is a random variable stochastically dominating the Beta distribution. We demonstrate the empirical effectiveness of our findings on synthetic and real datasets using a popular split conformal prediction procedure called conformalized quantile regression (CQR). 

\end{abstract}

\clearpage

\clearpage

\tableofcontents

\clearpage

\pagenumbering{arabic}

\section{Introduction}\label{sec:intro}

Black-box predictive models have become popular tools in the advent of large datasets and cheap computing resources. However, the predictive performance of these models is usually subject to weak statistical guarantees that only hold asymptotically in the size of the dataset or require strong parametric assumptions on the data generating process. Both might be unrealistic in practice, therefore the deployment of black-box predictive models is challenging in contexts where safety is key, such as medicine.  

A particular line of research aiming to improve this situation is called post-hoc calibration. Consider an arbitrary black-box predictor fitted on a proper training set, where the base predictor is black-box in the sense that we do not seek to understand or modify its behaviour, but instead wrap it into a larger post-hoc calibration algorithm. The purpose of such an algorithm is to calibrate the base predictor so that it satisfies some rigorous statistical guarantees under minimal assumptions. We specifically consider finite-sample guarantees while making no assumptions on the distribution of the underlying data, entering the field of distribution-free predictive inference.

Conformal prediction is a general framework to construct prediction sets that satisfy some distribution-free finite-sample guarantee under the assumption of iid data \cite{vovk2005algorithmic, vovk2009line}. The widely studied adaptation focused on in this thesis is called split conformal prediction \cite{lei2018distribution, papadopoulos2002inductive}. For a gentle introduction into conformal prediction we refer to \cite{angelopoulos2021gentle} and for a more technical tutorial to \cite{shafer2008tutorial}. Conformal prediction has been applied in various contexts, such as drug discovery \cite{alvarsson2021predicting}, image classification \cite{angelopoulos2020uncertainty}, natural language processing \cite{fisch2020efficient, fisch2021few, schuster2021consistent} and voting during the 2020 US presidential election \cite{cherian2020washington}. Although traditionally conformal prediction starts with the definition of a non-conformity score, we follow an alternative but equivalent interpretation leveraging nested prediction sets \cite{gupta2022nested}. 

Given the base predictor and a calibration set $\{(X_i,Y_i)\}^n_{i=1}$, we are interested in predicting the label $Y_{n+1}\in\mathcal{Y}$ corresponding to a new feature $X_{n+1}\in\mathcal{X}$, while quantifying the corresponding prediction uncertainty. Simply put, the split conformal prediction algorithm inputs the base predictor and the calibration set and outputs a prediction set $\mathcal{S}_{\widehat{\lambda}}(X_{n+1})$ that contains $Y_{n+1}$ with some distribution-free finite-sample guarantee of certainty. Note $\mathcal{S}_{\widehat{\lambda}}(X_{n+1})$ is indexed by a random variable $\widehat{\lambda}\in\Lambda$ that determines the size of the set, where $\Lambda\subset\mathbb{R}\cup\{\pm\infty\}$ is some closed set.

The most commonly used finite-sample guarantee in distribution-free predictive inference is that of marginal coverage, guaranteeing that a prediction set $\mathcal{S}_{\widehat{\lambda}}(X_{n+1})$ contains $Y_{n+1}$ with a pre-specified confidence level $1-\alpha\in(0,1)$ \textit{on average} over the iid sample $\{(X_i,Y_i)\}^{n+1}_{i=1}$. It is a well-known result what value of $\widehat{\lambda}$ yields marginal coverage in split conformal prediction \cite{lei2018distribution, papadopoulos2002inductive}. The guarantee of a tolerance region\footnote{Also known as training conditional validity \cite{vovk2012conditional} or probably approximately correct (PAC) coverage \cite{valiant1984theory, park2019pac, park2020pac}, the latter notion arising from statistical learning theory \cite{vapnik1999overview}.} slightly differs from marginal coverage, taking explicitly into account that the calibration set is random and thus that the coverage of $\mathcal{S}_{\widehat{\lambda}}(X_{n+1})$, conditional on the calibration set, is a random variable. $\mathcal{S}_{\widehat{\lambda}}(X_{n+1})$ is an $(\epsilon,\delta)$-tolerance region if it contains at least a pre-specified proportion $1-\epsilon\in(0,1)$ of the label population $\mathcal{Y}$ with at least a pre-specified probability $1-\delta\in(0,1)$ over the calibration data, see \cref{sec:delta_val} for a formal definition. The connection between marginal coverage and tolerance regions in the context of split conformal prediction is considered in \cite{vovk2012conditional}.

While conformal prediction has been pioneered in roughly the last two decades, procedures to construct tolerance regions based on order statistics have been studied extensively since the 1940s \cite{wilks1941determination, wald1943extension, scheffe1945non, tukey1947non, tukey1948nonparametric, fraser1951nonparametric, fraser1951sequentially, fraser1953nonparametric, kemperman1956generalized}. For a concise treatment we refer to \cite{guttman1970statistical} and for a more recent review to \cite{krishnamoorthy2009statistical}. \cite{vovk2012conditional} first studied the connection between these `classical tolerance predictors' and split conformal prediction, stating the property of prediction set validity in the traditional statistical language of classical tolerance predictors and interpreting split conformal prediction as a `conditional' version of the classical tolerance predictors proposed in \cite{wilks1941determination, scheffe1945non}. 

Exploring the more recent past, \cite{angelopoulos2022conformal, bates2021distribution, angelopoulos2021learn} introduced procedures to obtain finite-sample guarantees using a more general notion of statistical error than the probability of miscoverage, motivated by examples where miscoverage is not the natural notion of error. These `distribution-free risk control' algorithms open new avenues for post-hoc calibration, most notably by calibrating a base predictor leveraging upper confidence bounds on the unknown underlying risk function \cite{bates2021distribution} or through multiple hypothesis testing \cite{angelopoulos2021learn}, in the latter case moving into the direction of decision-making. Although the procedures are similar to split conformal prediction, they rely on entirely different proof techniques and the relation to split conformal prediction is not obvious.

\subsection{Summary and Outline}
This thesis reviews split conformal prediction, classical tolerance predictors and distribution-free risk control within the language of nested prediction sets, and explores the connections between these procedures. Although no new methodology is proposed, we present several novel insights. 

First, we reformulate known results regarding tolerance regions in the context of split conformal prediction in the language of nested prediction sets, specifying the value of $\widehat{\lambda}$ that results in an $(\epsilon,\delta)$-tolerance region. Furthermore, we expand on the duality between marginal coverage and tolerance regions, showing that a split conformal prediction set that satisfies marginal coverage is an $(\epsilon,\delta)$-tolerance region for certain $\epsilon,\delta$, and conversely, that a split conformal $(\epsilon,\delta)$-tolerance region automatically satisfies marginal coverage for certain $\alpha$. This is an extension of \cite[Proposition 2a-2b]{vovk2012conditional}, which only covers the former relation. 

Second, we elaborate on the crucial role of order statistics in the connection between classical tolerance predictors and split conformal prediction. In particular, we prove that the coverage of a split conformal prediction set, conditional on the calibration set, is a random variable stochastically dominating the Beta distribution. To the best of our knowledge, this result has only been hinted at in \cite{vovk2012conditional} and mentioned without proof in \cite{angelopoulos2021gentle} for label populations following a continuous distribution on $\mathcal{Y}$. Interestingly, this analytical distribution provides an alternative proof technique to obtain marginal coverage and tolerance regions in split conformal prediction.

Third, our focus with regard to distribution-free risk control lies in getting a better understanding of its relation to split conformal prediction. We show that conformal risk control (CRC) \cite{angelopoulos2022conformal} yields identical prediction sets to split conformal prediction for risks equal to the expectation of the 0-1 loss when the objective of interest is marginal coverage, although CRC applies more general to any bounded monotone risk function. Furthermore, upper confidence bound (UCB) calibration \cite{bates2021distribution} obtains identical prediction sets to split conformal prediction for risks equal to the expectation of the 0-1 loss when the objective of interest is tolerance regions, although UCB calibration applies more general to any monotone risk function. In turn, learn then test (LTT) \cite{angelopoulos2021learn} generalizes UCB calibration and thus split conformal prediction to any risk function. 

Finally, we verify our theoretical results by demonstrating how conformalized quantile regression (CQR) \cite{romano2019conformalized} can produce prediction sets with valid coverage, using quantile random forest (QRF) \cite{meinshausen2006quantile} as a base predictor. We apply CQR to synthetic and real datasets identical to the datasets used in \cite{romano2019conformalized} and demonstrate that any base predictor can be calibrated to produce valid prediction sets using split conformal prediction, independent of the accuracy of the base predictor. 

\cref{sec:guarantees} introduces notation and a brief outline of distribution-free predictive inference, as well as a formal definition of several finite-sample guarantees. Subsequently, \cref{sec:conformal} proceeds to a brief overview of the split conformal prediction procedure to obtain marginal coverage and tolerance regions, illustrated by CQR. This is followed by \cref{sec:dist_cov}, highlighting the connection between classical tolerance predictors and split conformal prediction. Afterwards, \cref{sec:riskcontrol} discusses distribution-free risk control, in particular the procedures proposed in \cite{angelopoulos2022conformal, bates2021distribution, angelopoulos2021learn}. \cref{sec:experiments} empirically verifies the earlier theoretical results through a case study of CQR. Finally, \cref{sec:conclusion} concludes and discusses further research suggestions.

\section{Distribution-Free Finite-Sample Guarantees}\label{sec:guarantees}

First, \cref{sec:pred_inf} provides a brief overview of distribution-free predictive inference and introduces some notation. Subsequently, \cref{sec:marg_cov} defines marginal coverage, \cref{sec:delta_val} defines tolerance regions and \cref{sec:risk_contr} generalizes these guarantees to a more general notion of statistical error.

\subsection{Distribution-Free Predictive Inference}\label{sec:pred_inf}

Consider an arbitrary base predictor fitted on a proper training set. We focus on post-hoc calibration throughout this thesis, such that we assume the proper training set and thus the base predictor to be fixed. In fact, the post-hoc calibration algorithms we discuss produce valid prediction sets for \textit{any given} base predictor. As such, all finite-sample guarantees and corresponding probability statements throughout this thesis are conditional on the proper training set.

The calibration set $\{(X_i,Y_i)\}^{n}_{i=1}$ of size $n$ consists of feature vectors $X_i\in\mathcal{X}$ and labels $Y_i\in\mathcal{Y}$ drawn from a population $\mathcal{X}\times\mathcal{Y}$. In general, the optimal split choice between proper training and calibration set depends on the accuracy of the base predictor, e.g. more complex procedures may require a larger proper training set. However, we do not pursue the optimal split choice in this thesis and assume an equal split. 

The base predictor and calibration set are leveraged by some post-hoc calibration method to construct a prediction set\footnote{Called `prediction interval' or `prediction region' in regression problems with respectively a one- or multi-dimensional population. Furthermore, we would like to point out the difference with the terms `confidence set/interval/region', which pertain to \textit{population parameters}.}  $\mathcal{S}_\lambda(X_{n+1})$ that contains the label $Y_{n+1}\in\mathcal{Y}$ corresponding to a new feature $X_{n+1}\in\mathcal{X}$ with some pre-specified guarantee of certainty. Technically, $\mathcal{S}_\lambda:\mathcal{X}\rightarrow\mathcal{Y}'$ is a set-valued random function that maps a feature vector $X_{n+1}\in\mathcal{X}$ to a set-valued prediction in some space of sets $\mathcal{Y}'$, e.g. $\mathcal{Y}'=2^\mathcal{Y}$. Since $\mathcal{S}_\lambda(X_{n+1})$ is random, we essentially train a functional estimator on the calibration set, but for simplicity we refer to $\mathcal{S}_\lambda(X_{n+1})$ as a prediction set.

The prediction set $\mathcal{S}_\lambda(X_{n+1})$ is indexed by a parameter $\lambda\in\Lambda$ that determines the size of the set, where $\Lambda\subset\mathbb{R}\cup\{\pm\infty\}$ is some closed set. This gives rise to a sequence of nested prediction sets $\{\mathcal{S}_\lambda(X_{n+1})\}_{\lambda\in\Lambda}$, which is fixed conditional on the proper training set. The sets are nested in the sense that larger values of $\lambda$ lead to larger prediction sets, i.e. 
\begin{equation}
    \forall \lambda_1\leq\lambda_2\in\Lambda \quad \Rightarrow \quad \mathcal{S}_{\lambda_1}(x)\subseteq \mathcal{S}_{\lambda_2}(x) \quad \forall x\in\mathcal{X}.
\end{equation}
Furthermore, the sequence is such that $\mathcal{S}_{\inf\Lambda}=\emptyset$ and $\mathcal{S}_{\sup\Lambda}=\mathcal{Y}$. This nested interpretation of prediction sets is inspired by \cite{gupta2022nested}. 

The key objective in distribution-free predictive inference is to infer $\widehat{\lambda}$, the smallest value of $\lambda$ such that $\mathcal{S}_{\widehat{\lambda}}(X_{n+1})$ contains $Y_{n+1}$ with some finite-sample guarantee under minimal distribution-free assumptions. The value of $\widehat{\lambda}$ is random through its dependence on the calibration set. We outline several finite-sample guarantees below. 

The underlying assumption often used in distribution-free predictive inference is that the calibration set $\{(X_i,Y_i)\}^{n}_{i=1}$ together with the new observation $(X_{n+1},Y_{n+1})$ is iid or exchangeable. We limit ourselves to iid data in this thesis, although many results discussed hold for exchangeable data as well. There are numerous recent advancements in extending conformal prediction beyond exchangeability, for example regarding covariate shift \cite{tibshirani2019conformal, park2021pac, qiu2022distribution}, label shift \cite{podkopaev2021distribution}, distribution shift \cite{cauchois2020robust, gibbs2021adaptive} and non-exchangeable data such as slowly-changing time-series \cite{barber2022conformal}. These have various applications regarding e.g. causal inference \cite{lei2021conformal, yin2022conformal, chernozhukov2021exact}, survival analysis \cite{candes2021conformalized}, private prediction sets \cite{angelopoulos2021private}, the design problem \cite{fannjiang2022conformal}, spatial data \cite{mao2020valid}, dependent data \cite{chernozhukov2018exact, dunn2018distribution, oliveira2022split} and off-policy evaluation \cite{taufiq2022conformal}.

\subsection{Marginal Coverage}\label{sec:marg_cov}

Perhaps the most basic finite-sample guarantee in distribution-free predictive inference is marginal coverage. A prediction set $\mathcal{S}_{\widehat{\lambda}}(X_{n+1})$ with marginal coverage at significance level $\alpha\in(0,1)$ contains the new label $Y_{n+1}$ with confidence level $1-\alpha$, on average over the randomness of the data.

\begin{definition}[Marginal Coverage]\label{def:marg_cov}
    A prediction set $\mathcal{S}_{\widehat{\lambda}}(X_{n+1})$ satisfies marginal coverage\footnote{Equivalently called `$\alpha$-tolerance predictor' \cite{vovk2005algorithmic} or `$1-\alpha$ expectation tolerance region' \cite{fraser1956nonparametric}.} with significance level $\alpha\in(0,1)$ if
    \begin{equation}\label{eq:marg_cov}
        \mathbb{P}\big[Y_{n+1}\in \mathcal{S}_{\widehat{\lambda}}(X_{n+1})\big]\geq 1-\alpha,
    \end{equation}
    for any sample size $n$ and any distribution on $\mathcal{X}\times\mathcal{Y}$.
\end{definition}
\begin{remark}
    Marginal coverage holds for any $n$ and is thus a finite-sample guarantee. Furthermore, the result holds for all distributions on $\mathcal{X}\times\mathcal{Y}$ and is thus distribution-free.
\end{remark}
\begin{remark}
    The probability holds marginally over the calibration set $\{(X_i,Y_i)\}^{n}_{i=1}$ and the new observation $(X_{n+1},Y_{n+1})$. This means that marginal coverage guarantees coverage \textit{on average} over all data $\{(X_i,Y_i)\}^{n+1}_{i=1}$. 
\end{remark}

We would like to point out the difference with the related but much stronger property of conditional coverage. While marginal coverage specifies a confidence level \textit{on average} over all $X_{n+1}\in\mathcal{X}$, conditional coverage specifies a confidence level \textit{for each} $X_{n+1}\in\mathcal{X}$. This can be formalized as
\begin{equation}\label{eq:cond_cov}
    \mathbb{P}\big[Y_{n+1}\in \mathcal{S}_{\widehat{\lambda}}(X_{n+1}) \ \big| \ X_{n+1}=x\big]\geq 1-\alpha,
\end{equation}
for almost all $x\in\mathcal{X}$, any sample size $n$ and any distribution on $\mathcal{X}\times\mathcal{Y}$. However, in regression problems conditional coverage is impossible to achieve by finite prediction sets without additional assumptions \cite{vovk2012conditional, lei2014distribution}. Under certain conditions, conditional coverage does hold asymptotically \cite{lei2018distribution}. Furthermore, there have been efforts for conditional coverage to approximately hold across regions of the feature space \cite{romano2019conformalized, izbicki2019flexible, guan2020conformal, romano2020classification, angelopoulos2020uncertainty, cauchois2021knowing, foygel2021limits}.

Similarly, class conditional coverage specifies a confidence level for subsets (`classes') of the label space $\mathcal{Y}$, such that coverage is approximately balanced across $\mathcal{Y}$. Various authors have demonstrated procedures to achieve (approximate) class conditional coverage \cite{lei2014classification, hechtlinger2018cautious, sadinle2019least, romano2019malice, guan2022prediction}. Although (class) conditional coverage is of great practical importance in many applications, it is outside of the scope of this thesis.

\subsection{Tolerance Regions}\label{sec:delta_val}

Marginal coverage holds on average over the calibration set $\{(X_i,Y_i)\}^n_{i=1}$ and the new observation $(X_{n+1},Y_{n+1})$. The key idea here is that the calibration set is a finite subset of the population and thus a random quantity, such that not all calibration sets yield coverage exactly equal to $1-\alpha$. The finite-sample guarantee of a tolerance region\footnote{We can distinguish between a `tolerance interval' and `tolerance region' when the population is respectively one- or multi-dimensional. We use the term tolerance region to address both.} is related to that of marginal coverage, but bounds the probability of a non-representative calibration set. A prediction set is a tolerance region if it is expected to contain at least a pre-specified proportion of the sampled population with a pre-specified probability.

\begin{definition}[$(\epsilon,\delta)$-Tolerance Region]\label{def:valid_delta}
    Let $\epsilon, \delta\in(0,1)$. A prediction set $\mathcal{S}_{\widehat{\lambda}}(X_{n+1})$ is an $(\epsilon, \delta)$-tolerance region\footnote{Equivalently called `$(\epsilon,\delta)$-validity' \cite{vovk2012conditional}, `probably approximately correct' (PAC) \cite{park2020pac}, `$(\epsilon, \delta)$-tolerance predictor' \cite{vovk2005algorithmic}, `$1-\delta$ tolerance region for a proportion $1-\epsilon$' \cite{fraser1956nonparametric}, or `$(1-\epsilon)$-content tolerance region at confidence $1-\delta$' \cite{fraser1956tolerance}.} if
    \begin{equation}\label{eq:valid_delta}
        \mathbb{P}\Big[\mathbb{P}\big[Y_{n+1}\in \mathcal{S}_{\widehat{\lambda}}(X_{n+1}) \ \big| \ \{(X_i,Y_i)\}^n_{i=1}\big] \geq 1-\epsilon \Big]\geq 1-\delta,
    \end{equation}
    for any sample size $n$ and any distribution on $\mathcal{X}\times\mathcal{Y}$. 
\end{definition}
An $(\epsilon,\delta)$-tolerance region automatically satisfies marginal coverage at significance level $\epsilon+\delta-\epsilon\delta$. This is apparent by writing the coverage probability of $\mathcal{S}_{\widehat{\lambda}}(X_{n+1})$ as
\begin{equation}
\begin{split}
    \mathbb{P}\big[Y_{n+1}\in \mathcal{S}_{\widehat{\lambda}}(X_{n+1})\big]&=\mathbb{E}\Big[\mathbb{P}\big[Y_{n+1}\in \mathcal{S}_{\widehat{\lambda}}(X_{n+1}) \ \big| \ \{(X_i,Y_i)\}^n_{i=1}\big]\Big]\\ &\geq(1-\epsilon)\mathbb{P}\Big[\mathbb{P}\big[Y_{n+1}\in \mathcal{S}_{\widehat{\lambda}}(X_{n+1}) \ \big| \ \{(X_i,Y_i)\}^n_{i=1}\big] \geq 1-\epsilon \Big]\\
    &\geq(1-\epsilon)(1-\delta)\\
    &=1-(\epsilon+\delta-\epsilon\delta),
\end{split}
\end{equation}
by subsequently applying the Law of Total Expectation, Markov's Inequality and the definition of an $(\epsilon,\delta)$-tolerance region. This lower bound on marginal coverage for $(\epsilon,\delta)$-tolerance regions holds for \textit{any} post-hoc calibration procedure to choose $\widehat{\lambda}$. In the case of split conformal prediction, we show a tighter bound in \cref{sec:tol_reg_split}.

\subsection{Risk Control}\label{sec:risk_contr}
Recently, \cite{angelopoulos2022conformal, bates2021distribution, angelopoulos2021learn} introduced finite-sample guarantees using a more general notion of statistical error than the coverage probability conditional on the calibration set, motivated by applications where the natural notion of error is not the rate of miscoverage. The objective is to control a risk $R:\Lambda\rightarrow\mathbb{R}$, corresponding to a prediction set $\mathcal{S}_\lambda(X_{n+1})$, where a higher risk indicates $\mathcal{S}_\lambda(X_{n+1})$ provides worse predictions of $Y_{n+1}$. If the sequence of nested prediction sets $\{\mathcal{S}_\lambda(X_{n+1})\}_{\lambda\in\Lambda}$ is fixed, then the risk is a deterministic function of $\lambda$ alone. The risk corresponding to $\widehat{\lambda}$ is dependent on the calibration set and thus random. 

The notion of marginal risk control (MRC) provides explicit bounds on $R(\widehat{\lambda})$ on average over the data $\{(X_i,Y_i)\}^{n+1}_{i=1}$, similar to marginal coverage. To the best of our knowledge, MRC has only appeared under the name of conformal risk control (CRC) \cite{angelopoulos2022conformal}, with the restriction that $R(\lambda)$ is the expectation of a bounded nested one-dimensional loss function. In turn, the notion of a risk controlling prediction (RCP) \cite{bates2021distribution, angelopoulos2021learn} provides a related but slightly different guarantee, similar to the notion of a tolerance region. 

\begin{definition}[Marginal Risk Control]\label{def:mrc}
    A prediction set $\mathcal{S}_{\widehat{\lambda}}(X_{n+1})$ satisfies marginal risk control (MRC) at upper bound level $\alpha<+\infty$ if 
    \begin{equation}\label{eq:crc}
        \mathbb{E}\big[R(\widehat{\lambda}) \big]\leq \alpha,
    \end{equation}
    for any sample size $n$ and any distribution on $\mathcal{X}\times\mathcal{Y}$.
\end{definition}
\begin{remark}\label{def:crc}
    If $R(\widehat{\lambda})=\mathbb{E}\big[L(Y_{n+1}, \mathcal{S}_{\widehat{\lambda}}(X_{n+1})) \ \big| \ \{(X_i,Y_i)\}^n_{i=1}\big]$, where $L:\mathcal{Y}\times\mathcal{Y}'\rightarrow(-\infty,B]$ is a bounded nested loss function for some $B<+\infty$, then this corresponds to the conformal risk control (CRC) guarantee mentioned in \cite{angelopoulos2022conformal}.
\end{remark}

\begin{definition}[$(\epsilon,\delta)$-Risk Controlling Prediction]\label{def:risk_control}
    Let $\epsilon, \delta\in(0,1)$. A prediction set $\mathcal{S}_{\widehat{\lambda}}(X_{n+1})$ is an $(\epsilon, \delta)$-risk controlling prediction (RCP) if
    \begin{equation}\label{eq:risk_control}
        \mathbb{P}\Big[R(\widehat{\lambda}) \leq \epsilon \Big]\geq 1-\delta,
    \end{equation}
    for any sample size $n$ and any distribution on $\mathcal{X}\times\mathcal{Y}$.
\end{definition}
\begin{remark}\label{rem:rcps}
    If $R(\widehat{\lambda})=\mathbb{E}\big[L(Y_{n+1},\mathcal{S}_{\widehat{\lambda}}(X_{n+1}))\ \big| \ \{(X_i,Y_i)\}^n_{i=1}\big]$, where $L:\mathcal{Y}\times\mathcal{Y}'\rightarrow\mathbb{R}_{\geq0}$ is a nested loss function, then this corresponds to the $(\epsilon, \delta)$-risk controlling prediction set (RCPS) guarantee mentioned in \cite{bates2021distribution}.
\end{remark}

\begin{remark}
    The parameters $\alpha$ in MRC and $\epsilon$ in an $(\epsilon, \delta)$-RCP do not necessarily represent a probability as in the case of marginal coverage and tolerance regions. Their magnitude depends upon the scale of the risk function. 
\end{remark}

Oftentimes we take the risk to be the expectation of a continuous nested loss function $L:\mathcal{Y}\times\mathcal{Y}'\rightarrow\mathbb{R}_{\geq0}$, but risk control is not limited to this particular case or even one-dimensional risks. The loss function is nested in a sense that larger prediction sets lead to smaller losses, i.e. for arbitrary $x\in\mathcal{X}$, $y\in\mathcal{Y}$,
\begin{equation}\label{eq:mon_cond}
    \forall \lambda_1\leq\lambda_2\in\Lambda \quad \Rightarrow \quad \mathcal{S}_{\lambda_1}(x)\subseteq \mathcal{S}_{\lambda_2}(x) \quad \Rightarrow \quad L(y,\mathcal{S}_{\lambda_1}(x))\geq L(y,\mathcal{S}_{\lambda_2}(x)).
\end{equation}

We can recover the familiar notion of coverage, conditional on the calibration set, by specifying the risk as the conditional expectation of the 0-1 loss, i.e.
\begin{equation}\label{eq:01loss}
    \begin{split}
    R_{0-1}(\widehat{\lambda})&=\mathbb{E}\Big[\mathbbm{1}\{Y_{n+1}\not\in\mathcal{S}_{\widehat{\lambda}}(X_{n+1})\}\ \big| \ \{(X_i,Y_i)\}^n_{i=1}\Big] \\
    &= 1- \mathbb{P}\big[Y_{n+1}\in\mathcal{S}_{\widehat{\lambda}}(X_{n+1})\ \big| \ \{(X_i,Y_i)\}^n_{i=1}\big].
    \end{split}
\end{equation}
For risks as in \cref{eq:01loss}, MRC reduces to marginal coverage and an $(\epsilon,\delta)$-RCP reduces to an $(\epsilon,\delta)$-tolerance region.

\section{Conformal Prediction}\label{sec:conformal}

\cref{sec:conf_pred} introduces split conformal prediction and specifically the nested set interpretation of split conformal prediction. Afterwards, \cref{sec:split} and \cref{sec:tol_reg_split} show how to construct split conformal prediction sets that result in marginal coverage and tolerance regions respectively. Finally, \cref{sec:cqr} demonstrates a specific split conformal regression procedure called conformalized quantile regression (CQR).

\subsection{Split Conformal Prediction}\label{sec:conf_pred}
Conformal prediction is a general framework to construct prediction sets that satisfy some distribution-free finite-sample guarantee under the assumption of iid data. The original specification is called full conformal prediction \cite{vovk2005algorithmic, vovk2009line}. A more computationally efficient and widely studied adaptation that splits the available data into a proper training and calibration set is split conformal prediction\footnote{Originally called `inductive conformal prediction' \cite{papadopoulos2002inductive}.} \cite{papadopoulos2002inductive, lei2018distribution}. However, a clear separation of the available data implies that the calibration set is not used to fit the base predictor and similarly, the proper training set is not used in post-hoc calibration. Several alternatives exist that address this issue, including cross conformal prediction \cite{vovk2015cross}, jackknife+ \cite{barber2021predictive} and out-of-bag conformal \cite{johansson2014regression, bostrom2017accelerating, linusson2020efficient, kim2020predictive}. However, we do not engage with these recent developments.

The traditional definition of split conformal prediction starts with a non-conformity score $r:\mathcal{X}\times\mathcal{Y}\rightarrow\mathbb{R}$, fully determining the conformal prediction procedure. The non-conformity score measures how well an observation $(x,y)$ conforms to the proper training set, where larger values indicate less conformity. Recently, \cite{gupta2022nested} proposed an alternative but equivalent view of conformal prediction that starts with the sequence of nested prediction sets $\{\mathcal{S}_\lambda(X_{n+1})\}_{\lambda\in\Lambda}$, from which the non-conformity score function $r(x,y)$ follows naturally. We follow this nested interpretation here. 

The nested sequence of prediction sets $\{\mathcal{S}_\lambda(X_{n+1})\}_{\lambda\in\Lambda}$ is constructed through the base predictor and considered fixed. However, there are various design choices for the sequence similar to the choice of non-conformity score in the traditional definition of conformal prediction. \cite[Table 1]{gupta2022nested} contains several examples of $\{\mathcal{S}_\lambda(X_{n+1})\}_{\lambda\in\Lambda}$ in various settings and both adaptive and symmetric expansions around the base predictor. From the sequence of prediction sets follows the non-conformity score function
\begin{equation}\label{eq:score}
    r(x,y):=\inf\{\lambda\in\Lambda:y\in\mathcal{S}_\lambda(x)\},
\end{equation}
i.e. the smallest $\lambda$ such that $y$ is included in $\mathcal{S}_\lambda(x)$. 

Now consider the fresh iid calibration set $\{(X_i,Y_i)\}^n_{i=1}$. We are interested in $\widehat{\lambda}_\mathrm{split}$, the smallest value of $\lambda$ such that the \textit{split conformal prediction set} $\mathcal{S}_{\widehat{\lambda}_\mathrm{split}}(X_{n+1})$ satisfies some distribution-free finite-sample guarantee. Using the non-conformity score function,
\begin{equation}\label{eq:split_set}
    \mathcal{S}_{\widehat{\lambda}_\mathrm{split}}(X_{n+1})=\{y\in\mathcal{Y}:r(X_{n+1},y)\leq \widehat{\lambda}_\mathrm{split}\},
\end{equation}
such that $\widehat{\lambda}_\mathrm{split}$ can be interpreted as an upper bound on the score function.

\subsection{Marginal Coverage and Split Conformal Prediction}\label{sec:split}

We are interested in the value of $\widehat{\lambda}_\mathrm{split}$ that results in marginal coverage at significance level $\alpha\in(0,1)$. To this regard, we take a closer look at the non-conformity score function (\cref{eq:score}). The set of non-conformity scores corresponding to the calibration set is $\{r_i\}^n_{i=1}$, where $r_i:=r(X_i,Y_i)$. 

For a pre-specified significance level $\alpha\in(0,1)$, we identify the quantile
\begin{equation}\label{eq:emp_quan}
    \widehat{Q}(\alpha):=\frac{1}{n}\lceil (1-\alpha)(n+1)\rceil\text{-th quantile of }\{r_i\}^n_{i=1},
\end{equation}
which is close to the $(1-\alpha)$-th quantile with a slight finite-sample correction. We can equivalently define $\widehat{Q}(\alpha)$ using order statistics as $\widehat{Q}(\alpha)=r_{(\lceil (1-\alpha)(n+1)\rceil)}$, i.e. the $\lceil (1-\alpha)(n+1)\rceil$-th largest element of $\{r_i\}^n_{i=1}$. This will become relevant in \cref{sec:dist_cov}. 

The proposition below, originally shown by \cite{papadopoulos2002inductive, lei2018distribution}, summarizes that $\widehat{\lambda}_\mathrm{split}=\widehat{Q}(\alpha)$ results in marginal coverage at significance level $\alpha$. The key idea behind this result is that if the sample $\{(X_i,Y_i)\}^{n+1}_{i=1}$ is iid, then the rank of the non-conformity score $r_{n+1}$ among $r_1,\dots,r_n,r_{n+1}$ is uniform over the set $\{1,\dots,n+1\}$.

\begin{proposition}\label{prop:marg_cov_split}
    Suppose the sample $\{(X_i,Y_i)\}^{n+1}_{i=1}$ is iid. Then the split conformal prediction set $\mathcal{S}_{\widehat{\lambda}_\mathrm{split}}(X_{n+1})$ (\cref{eq:split_set}) with $\widehat{\lambda}_\mathrm{split}=\widehat{Q}(\alpha)$ (\cref{eq:emp_quan}) satisfies marginal coverage with significance level $\alpha\in(0,1)$, i.e.
    \begin{equation}\label{eq:marg_cov_split}
        \mathbb{P}\big[Y_{n+1}\in \mathcal{S}_{\widehat{\lambda}_\mathrm{split}}(X_{n+1})\big]\geq 1-\alpha.
    \end{equation}
    Moreover, if the non-conformity scores are almost surely distinct, then the split conformal prediction set is nearly perfectly calibrated, i.e. 
    \begin{equation}\label{eq:marg_cov_split_as}
        \mathbb{P}\big[Y_{n+1}\in \mathcal{S}_{\widehat{\lambda}_\mathrm{split}}(X_{n+1})\big]\leq 1-\alpha+\frac{1}{n+1}.
    \end{equation}
\end{proposition}

\begin{proof}
    See \cref{app:prop1}.
\end{proof}

\begin{remark}
    No assumptions are made on the distribution of the proper training set, such that it does not need to come from the same distribution as the rest of the data. Furthermore, the result holds both conditional and unconditional on the proper training set, the former implying that split conformal prediction is able to calibrate \textit{any given} base predictor. 
\end{remark}

\subsection{Tolerance Regions and Split Conformal Prediction}\label{sec:tol_reg_split}

Instead of marginal coverage, suppose we are interested in the value of $\widehat{\lambda}_\mathrm{split}$ that results in an $(\epsilon,\delta)$-tolerance region at significance level $\epsilon,\delta\in(0,1)$. First, we introduce some notation. Throughout this thesis, $\mathrm{Bin}(k; m, p)$ denotes the Binomial cumulative distribution function with $m\geq 1$ trials and success probability $p\in(0,1)$, evaluated at a certain number of successes $0\leq k\leq m$. For pre-specified $\epsilon,\delta\in(0,1)$, we identify the quantile
\begin{equation}\label{eq:P_hat}
    \widehat{P}(\epsilon,\delta):=\Big(1-\frac{1}{n}\sup\big\{k:\mathrm{Bin}(k;n,\epsilon)\leq\delta\big\}\Big)\text{-th quantile of }\{r_i\}^n_{i=1}.
\end{equation}
We can equivalently define $\widehat{P}(\epsilon,\delta)$ using order statistics as $\widehat{P}(\epsilon,\delta)=r_{(n-\sup\{k:\mathrm{Bin}(k;n,\epsilon)\leq\delta\})}$, i.e. the $\Big(n-\sup\big\{k:\mathrm{Bin}(k;n,\epsilon)\leq\delta\big\}\Big)$-th largest element of $\{r_i\}^n_{i=1}$. This will become relevant in \cref{sec:dist_cov}. In practice, a numerical method is needed to invert the Binomial distribution. We include tables with conventional values of $\epsilon,\delta$ and $n$ in \cref{sec:tables}. 

\cshref{prop:tol_reg_split} below summarizes that $\widehat{\lambda}_\mathrm{split}=\widehat{P}(\epsilon,\delta)$ results in an $(\epsilon,\delta)$-tolerance region for iid data. To provide some intuition, the split conformal prediction set $\mathcal{S}_{\widehat{\lambda}_\mathrm{split}}(X_i)$ does not cover the label $Y_i$ if and only if $r(X_i,Y_i)>\widehat{\lambda}_\mathrm{split}$. If the sample $\{(X_i,Y_i)\}^{n+1}_{i=1}$ is iid, then choosing $\widehat{\lambda}_\mathrm{split}$ based on the calibration set corresponds to the success probability of $n$ independent Bernoulli trials $r(X_i,Y_i)>\widehat{\lambda}_\mathrm{split}$ and thus a Binomial distribution. See the proof of \cshref{prop:split_mc_tol} for more details. 

\begin{proposition}\label{prop:tol_reg_split}
    Let $\epsilon, \delta\in(0,1)$ and suppose the sample $\{(X_i,Y_i)\}^{n+1}_{i=1}$ is iid. Then the split conformal prediction set $\mathcal{S}_{\widehat{\lambda}_\mathrm{split}}(X_{n+1})$ (\cref{eq:split_set}) with $\widehat{\lambda}_\mathrm{split}=\widehat{P}(\epsilon,\delta)$ (\cref{eq:P_hat}) is an $(\epsilon,\delta)$-tolerance region, i.e.
    \begin{equation}
        \mathbb{P}\Big[\mathbb{P}\big[Y_{n+1}\in \mathcal{S}_{\widehat{\lambda}_\mathrm{split}}(X_{n+1}) \ \big| \ \{(X_i,Y_i)\}^n_{i=1}\big] \geq 1-\epsilon \Big]\geq 1-\delta.
    \end{equation}
    Moreover, if the non-conformity scores are almost surely distinct, then $\mathcal{S}_{\widehat{\lambda}_\mathrm{split}}(X_{n+1})$ is an $(\epsilon,\delta)$-tolerance region if and only if $\widehat{\lambda}_\mathrm{split}$ is at least $\widehat{P}(\epsilon,\delta)$. 
\end{proposition}
\begin{proof}
    See \cref{app:prop2}.
\end{proof}

As mentioned in \cref{sec:delta_val}, tolerance regions and marginal coverage are closely related. In particular for split conformal prediction, a prediction set that satisfies marginal coverage at level $\alpha$ automatically implies an $(\epsilon,\delta)$-tolerance region for certain $\epsilon,\delta$. Conversely, an $(\epsilon,\delta)$-tolerance region automatically implies marginal coverage at a certain significance level $\alpha$. This is summarized in the proposition below, providing an extension on \cite[Proposition 2a-2b]{vovk2012conditional}, which only covers the former relation.

\begin{proposition}\label{prop:split_mc_tol}
    Let $\alpha,\epsilon, \delta\in(0,1)$ and suppose the sample $\{(X_i,Y_i)\}^{n+1}_{i=1}$ is iid.
    \begin{itemize}
        \item[(i)] The split conformal prediction set $\mathcal{S}_{\widehat{\lambda}_\mathrm{split}}(X_{n+1})$ (\cref{eq:split_set}) with $\widehat{\lambda}_\mathrm{split}=\widehat{Q}(\alpha)$ (\cref{eq:emp_quan}) is an $(\epsilon,\delta)$-tolerance region if
    \begin{equation}\label{eq:delta_cond}
        \delta \geq \mathrm{Bin}(\lfloor \alpha(n+1)-1\rfloor; n,\epsilon),
    \end{equation}
    or equivalently
    \begin{equation}\label{eq:eps_cond}
        \epsilon \geq \inf\big\{p: \mathrm{Bin}(\lfloor \alpha(n+1)-1\rfloor; n, p)\leq \delta\big\}.
    \end{equation}
    Moreover, if the non-conformity scores are almost surely distinct, then the above becomes an `if and only if' statement.
    \item[(ii)] The split conformal prediction set $\mathcal{S}_{\widehat{\lambda}_\mathrm{split}}(X_{n+1})$ (\cref{eq:split_set}) with $\widehat{\lambda}_\mathrm{split}=\widehat{P}(\epsilon,\delta)$ (\cref{eq:P_hat}) satisfies marginal coverage at confidence level equal to
    \begin{equation}
        \mathbb{P}[Y_{n+1}\in \mathcal{S}_{\widehat{\lambda}_\mathrm{split}}(X_{n+1})]\geq1-\frac{1}{n+1}\big(\sup\big\{k: \mathrm{Bin}(k; n, \epsilon)\leq\delta\big\}+1\big)
    \end{equation}
    Moreover, if the non-conformity scores are almost surely distinct, then the split conformal prediction set is nearly perfectly calibrated, i.e.
    \begin{equation}
        \mathbb{P}[Y_{n+1}\in \mathcal{S}_{\widehat{\lambda}_\mathrm{split}}(X_{n+1})]\leq1-\frac{1}{n+1}\big(\sup\big\{k: \mathrm{Bin}(k; n, \epsilon)\leq\delta\big\}\big).
    \end{equation}
    \end{itemize}
\end{proposition}
\begin{proof}
    See \cref{app:prop_new}.
\end{proof}

\subsection{Conformalized Quantile Regression}\label{sec:cqr}
To illustrate the split conformal prediction algorithm, we provide an example for regression, i.e. $\mathcal{Y}=\mathbb{R}$. Recently, quantile regression \cite{koenker1978regression} has proven to be an efficient choice of base predictor for split conformal prediction in the context of heteroskedastic data, coined conformalized quantile regression (CQR) \cite{romano2019conformalized}. The nested interpretation of CQR given below is inspired by \cite{gupta2022nested}.

First, we consider the base predictor. For a pre-specified $\alpha\in(0,1)$, estimation of the conditional quantile functions $q_{\alpha/2}(\cdot)$ and $q_{1-\alpha/2}(\cdot)$ on the proper training set yields estimates $\hat{q}_{\alpha/2}(\cdot)$ and $\hat{q}_{1-\alpha/2}(\cdot)$ respectively. For simplicity, we assume $\hat{q}_{\alpha/2}(x)\leq \hat{q}_{1-\alpha/2}(x)$ for all $x\in\mathcal{X}$ to avoid the quantile crossing problem \cite{bassett1982empirical}. Under certain conditions, the conditional quantile estimators are known to be consistent \cite{takeuchi2006nonparametric, steinwart2011estimating}, such that the interval $\big[\hat{q}_{\alpha/2}(\cdot), \hat{q}_{1-\alpha/2}(\cdot)\big]$ asymptotically satisfies the desired coverage of $1-\alpha$. However, this is not guaranteed in finite samples. 

To achieve valid coverage in finite-samples, we turn to split conformal prediction. Consider a symmetric increase or reduction around the base predictor to construct the sequence of nested prediction sets $\{\mathcal{S}_\lambda(X_{n+1})\}_{\lambda\in\Lambda}$, i.e. for a specific $\lambda\in\Lambda\subset\mathbb{R}\cup\{\pm\infty\}$ we have
\begin{equation}
    \mathcal{S}_\lambda(x)=\big[\hat{q}_{\alpha/2}(x) - \lambda, \hat{q}_{1-\alpha/2}(x) + \lambda\big].
\end{equation}
Note that larger $\lambda$ lead to larger prediction sets under the assumption made earlier that $\hat{q}_{\alpha/2}(x)\leq \hat{q}_{1-\alpha/2}(x)$ for all $x\in\mathcal{X}$. Now the non-conformity score (\cref{eq:score}) becomes
\begin{equation}\label{eq:cqr_score}
    \begin{split}
        r(x,y)&= \inf\{\lambda\in\mathbb{R}:y\in\mathcal{S}_\lambda(x)\}\\
        &= \inf\{\lambda\in\mathbb{R}:\hat{q}_{\alpha/2}(x) - \lambda\leq y \leq \hat{q}_{1-\alpha/2}(x) + \lambda\}\\
        &= \sup\{\hat{q}_{\alpha/2}(x) - y, y - \hat{q}_{1-\alpha/2}(x)\},
    \end{split}
\end{equation}
i.e. the smallest $\lambda$ such that $y$ is contained in $\mathcal{S}_\lambda(x)$. We can interpret the score function as follows. If $y$ lies outside the interval $\big[\hat{q}_{\alpha/2}(x), \hat{q}_{1-\alpha/2}(x)\big]$, then the score $r(x,y)$ is positive to account for undercoverage. If $y$ lies correctly inside the interval $\big[\hat{q}_{\alpha/2}(x), \hat{q}_{1-\alpha/2}(x)\big]$, then the score $r(x,y)$ is non-positive to account for overcoverage. 

The set of non-conformity scores $\{r_i\}^n_{i=1}$ corresponding to the calibration set satisfies $r_i:= r(X_i,Y_i)=\max\{\hat{q}_{\alpha/2}(X_i) - Y_i, Y_i - \hat{q}_{1-\alpha/2}(X_i)\}$. For a pre-specified significance level $\alpha\in(0,1)$ and $\widehat{\lambda}_\mathrm{CQR}=\widehat{Q}(\alpha)$, the resulting \textit{CQR prediction set} is
\begin{equation}\label{eq:cqr_set}
    \mathcal{S}_{\widehat{\lambda}_\mathrm{split}}(X_{n+1}) = \big[\hat{q}_{\alpha/2}(X_{n+1}) - \widehat{Q}(\alpha), \hat{q}_{1-\alpha/2}(X_{n+1}) + \widehat{Q}(\alpha)\big].
\end{equation}
By \cshref{prop:marg_cov_split}, we know that if the sample $\{(X_i,Y_i)\}^{n+1}_{i=1}$ is iid, then the above CQR prediction set satisfies marginal coverage at level $\alpha$. By \cshref{prop:tol_reg_split}, we know that if we replace $\widehat{Q}(\alpha)$ with $\widehat{P}(\epsilon,\delta)$ in \cref{eq:cqr_set}, then the CQR prediction set is an $(\epsilon,\delta)$-tolerance region. By \cshref{prop:split_mc_tol}, we know that specifying $\alpha$ results in a tolerance region for certain $\epsilon,\delta$ and vice versa, specifying $\epsilon,\delta$ results in marginal coverage for a certain $\alpha$.

\section{Distribution of Coverage}\label{sec:dist_cov}

This section focuses on the probability of coverage of a new label $Y_{n+1}$, conditional on the calibration set, corresponding to a split conformal prediction set $\mathcal{S}_{\widehat{\lambda}_\mathrm{split}}(X_{n+1})$. In particular, the distribution of this random variable and its relation to classical tolerance predictors of the 1940s. This shows an alternative perspective to the results in \cref{sec:conformal}. We would like to point out that \cref{sec:dis_cov_exp} and \cref{sec:dist_cov_fin_sample} are a generalization of the results in \cite[Section 3.1-3.2]{angelopoulos2021gentle}, removing the assumption of almost surely distinct non-conformity scores.

\cref{sec:tol_reg} starts with a historic overview of the rich literature regarding classical tolerance predictors. Then, \cref{sec:wilks} outlines the simplest of these predictors and the analytical distribution of the coverage probability of the resulting prediction set. Subsequently, \cref{sec:dis_cov_exp} connects classical tolerance predictors and the analytical distribution of the coverage probability to split conformal prediction. Finally, \cref{sec:dist_cov_fin_sample} extends on the analytical results in the setting of a finite test set.

\subsection{Classical Tolerance Predictors}\label{sec:tol_reg}

While conformal prediction has been pioneered by Vladimir Vovk and colleagues in roughly the last two decades \cite{vovk2005algorithmic, vovk2009line}, the first important paper on classical tolerance predictors was written by Wilks in 1941 \cite{wilks1941determination} and the theory further extended throughout the 1940s and early 1950s. Wilks constructed classical tolerance predictors for a continuous univariate population, which was extended to the discontinuous case by \cite{scheffe1945non}. We provide more details on these tolerance predictors in \cref{sec:wilks}. More than half a century later, \cite[Appendix A]{vovk2012conditional} showed how to interpret split conformal prediction as a `conditional' version of these classical tolerance predictors. We extend on this in \cref{sec:dis_cov_exp}. 

The early work on classical tolerance predictors was mainly devoted to the construction of `distribution-free' tolerance predictors. Distribution-free in this context means that the distribution of the coverage of the tolerance predictor, conditional on the observed sample, is independent of the distribution of the sampled population. Distribution-free tolerance predictors are based on order statistics, using the key result from \cite{wilks1941determination} that given a sample from a continuous univariate distribution, the coverage of an interval bounded by an order statistic on each side has a distribution dependent only on those particular order statistics and thus independent of the sampled population. We expand on this in \cref{sec:wilks}. 

Wilks' classical tolerance predictors were extended to multivariate populations by Wald in 1943 \cite{wald1943extension}. In turn, Tukey provided an important generalization to tolerance predictors of any general shape in 1947 \cite{tukey1947non}. The assumption of a continuous population was then removed by \cite{tukey1948nonparametric, fraser1951nonparametric} and the setting of distribution-free tolerance predictors further extended by \cite{fraser1951sequentially, fraser1953nonparametric, kemperman1956generalized}. Interestingly, \cite[p. 257]{vovk2005algorithmic} showed that Tukey's classical tolerance predictors can be interpreted as a special case of \textit{full} conformal prediction \cite{vovk2005algorithmic, vovk2009line} (not \textit{split} conformal prediction).

\subsection{Simple Tolerance Predictors of Order Statistics}\label{sec:wilks}

To introduce some intuition behind classical tolerance predictors, we outline Wilks' classical tolerance predictor \cite{wilks1941determination} for a continuous univariate population on $\mathbb{R}$ here in full detail. This special case is contained in the more general result including discontinuous univariate populations in \cite{scheffe1945non}. \cshref{prop:wilks_cov} shows the analytical distribution of the probability of coverage of a tolerance region based on order statistics for a univariate population, given the observed sample. This includes the discontinuous case, but we defer the proof to \cite[p. 192]{scheffe1945non}.

Let the feature space $\mathcal{X}$ be a one-element set and suppose a sample $\{Y_i\}^n_{i=1}$ of $n$ independent labels $Y_i\in\mathcal{Y}$ is taken from a continuous univariate distribution $F_Y$. Ordering the $Y_i$ in ascending order yields a sample of order statistics $\{Y_{(i)}\}^n_{i=1}$ such that $Y_{(1)}<\dots<Y_{(n)}$. In addition, define $Y_{(0)}=-\infty$ and $Y_{(n+1)}:=+\infty$. Note that the $Y_{(i)}$ are almost surely distinct since we are sampling from a continuous population and thus $\mathbb{P}\big[Y_{(i)}=Y_{(j)}\big]=0$ for all $i,j$.

Wilks' procedure simply consists of constructing the prediction interval $Y_{(r)}\leq Y_{n+1}\leq Y_{(s)}$ for some $0\leq r<s\leq n+1$ based on the sample of order statistics, while choosing $r$ and $s$ such that $Y_{(r)}\leq Y_{n+1}\leq Y_{(s)}$ is an $(\epsilon,\delta)$-tolerance region. For simplicity, we exclude the case that both $r=0$ and $s=n+1$, since a new label $Y_{n+1}$ is contained in the interval $\big[Y_{(0)},Y_{(n+1)}\big]=\big(-\infty,+\infty\big)$ almost surely and the prediction interval is non-informative. 

Let $\mathrm{Beta}(p; m, k)$ denote the regularized incomplete Beta function with $p\in(0,1)$ and parameters $m,k>0$. The prediction interval $Y_{(r)}\leq Y_{n+1}\leq Y_{(s)}$ is an $(\epsilon,\delta)$-tolerance region if and only if $r$ and $s$ are chosen such that
\begin{equation}
    \mathrm{Beta}(1-\epsilon; s-r, n-s+r+1)\leq\delta.
\end{equation}
This follows directly from the following result on the distribution of the coverage probability, conditional on the calibration set, showing that it is independent of the distribution of the sampled population and that it equals a random variable stochastically dominating a Beta distribution, with equality if the $Y_i$ are almost surely distinct. 

\begin{proposition}[Distribution of Coverage for Univariate Populations]\label{prop:wilks_cov}
    Let $\{Y_{i}\}^{n+1}_{i=1}$ be iid from a univariate population. Then for $0\leq r<s\leq n+1$, the probability of coverage of the prediction interval $Y_{(r)}\leq Y_{n+1}\leq Y_{(s)}$, conditional on the observed sample, satisfies 
    \begin{equation}\label{eq:wilks_cov_disc}
        \mathbb{P}\big[Y_{(r)}\leq Y_{n+1}\leq Y_{(s)} \ \big|\  \{Y_{i}\}^n_{i=1}\big] \geq Z,
    \end{equation}
    where 
    \begin{equation}
        Z \sim \mathrm{Beta}(s-r, n-s+r+1),
    \end{equation}
    with equality if the $Y_i$ are almost surely distinct.
\end{proposition}

\begin{proof}
    See \cref{app:prop3}.
\end{proof}

\subsection{Distribution of Coverage for Split Conformal Prediction}\label{sec:dis_cov_exp}

Now we relate the analytical distribution of coverage for a univariate population to split conformal prediction, leveraging the idea mentioned in \cite[Appendix A]{vovk2012conditional} to interpret split conformal prediction as a `conditional' version of Wilks' classical tolerance predictor for a univariate population. Recall the set of non-conformity scores corresponding to the calibration set $\{r_i\}^n_{i=1}$, where $r_i=r(X_i,Y_i)$. If the sample $\{(X_i,Y_i)\}^{n+1}_{i=1}$ is iid, then the non-conformity scores can be interpreted as an iid sample from a univariate population and we can formulate a prediction set based on order statistics similar to \cref{sec:wilks}. 

The set of order statistics corresponding to the non-conformity scores is $\{r_{(i)}\}^n_{i=1}$ such that $r_{(1)}\leq\dots\leq r_{(n)}$. We further define $r_{(0)}:=-\infty$ and $r_{(n+1)}:=+\infty$. Now consider the prediction interval $\big[r_{(0)}, r_{(\lceil (1-\alpha)(n+1)\rceil)}\big]=\big(-\infty,r_{(\lceil (1-\alpha)(n+1)\rceil)}\big]$ and recall that $\widehat{Q}(\alpha)=r_{(\lceil (1-\alpha)(n+1)\rceil)}$. Then we can write the split conformal prediction set $\mathcal{S}_{\widehat{\lambda}_\mathrm{split}}(X_{n+1})$ (\cref{eq:split_set}) with $\widehat{\lambda}_\mathrm{split}=\widehat{Q}(\alpha)$ (\cref{eq:emp_quan}) equivalently as 
\begin{equation}\label{eq:split_set_score}
    \mathcal{S}_{\widehat{\lambda}_\mathrm{split}}(X_{n+1}) =\big\{y\in\mathcal{Y}:-\infty<r(X_{n+1},y)\leq r_{(\lceil (1-\alpha)(n+1)\rceil)}\big\}.
\end{equation}
From \cshref{prop:wilks_cov} we know that the probability of coverage of the prediction interval $-\infty<r(X_{n+1},Y_{n+1})\leq r_{(\lceil (1-\alpha)(n+1)\rceil)}$, conditional on the observed scores $\{r_i\}^n_{i=1}$, is a random variable stochastically dominating a Beta distribution, with equality if the non-conformity scores are almost surely distinct. This gives the following result for the distribution of the probability of coverage, conditional on the calibration set, corresponding to the split conformal prediction set $\mathcal{S}_{\widehat{\lambda}_\mathrm{split}}(X_{n+1})$.

\begin{proposition}[Distribution of Coverage]\label{prop:dist_exp_cov}
    Let $\alpha\in(0,1)$ and suppose the sample $\{(X_i,Y_i)\}^{n+1}_{i=1}$ is iid. Then the probability of coverage of the split conformal prediction set $\mathcal{S}_{\widehat{\lambda}_\mathrm{split}}(X_{n+1})$ (\cref{eq:split_set}) with $\widehat{\lambda}_\mathrm{split}=\widehat{Q}(\alpha)$ (\cref{eq:emp_quan}), conditional on the calibration set, satisfies
    \begin{equation}\label{eq:dist_cov_exp_nonc}
        \mathbb{P}\big[Y_{n+1}\in \mathcal{S}_{\widehat{\lambda}_\mathrm{split}}(X_{n+1}) \ \big| \ \{(X_i,Y_i)\}^n_{i=1}\big]\geq Z,
    \end{equation}
    where
    \begin{equation}
        Z\sim\mathrm{Beta}(\lceil (1-\alpha)(n+1)\rceil,\lfloor \alpha(n+1)\rfloor),
    \end{equation}
    with equality if the non-conformity scores are almost surely distinct.

\end{proposition}

\begin{proof}
    See \cref{app:prop4}.
\end{proof}
\begin{remark}\label{rem:beta_P_hat}
    If instead $\widehat{\lambda}_\mathrm{split}=\widehat{P}(\epsilon,\delta)$ (\cref{eq:P_hat}), then under the same assumptions as \cshref{prop:dist_exp_cov}, the split conformal prediction set  $\mathcal{S}_{\widehat{\lambda}_\mathrm{split}}(X_{n+1})$ satisfies \cref{eq:dist_cov_exp_nonc} where
    \begin{equation}
        Z\sim\mathrm{Beta}\Big(n-\sup\big\{k: \mathrm{Bin}(k; n, \epsilon)\leq\delta\big\}, \sup\big\{k: \mathrm{Bin}(k; n, \epsilon)\leq\delta\big\}+1\Big).
    \end{equation}
    This becomes apparent by observing that $\widehat{P}(\epsilon,\delta)$ equals $\widehat{Q}(\alpha)$ for all $\alpha$ in the interval 
    \begin{equation}
        \frac{1}{n+1}\big(\sup\big\{k: \mathrm{Bin}(k; n, \epsilon)\leq\delta\big\}+1\big)\leq\alpha< \frac{1}{n+1}\big(\sup\big\{k: \mathrm{Bin}(k; n, \epsilon)\leq\delta\big\}+2\big),
    \end{equation}
    and thus in particular the smallest value in this interval. Note that the strict inequality follows from the ceiling function in $\widehat{Q}(\alpha)$.
\end{remark}

An interesting property of \cshref{prop:dist_exp_cov} is that it contains \cshref{prop:marg_cov_split}, \cshref{prop:tol_reg_split} and \cshref{prop:split_mc_tol}, providing an alternative, although more elaborate, proof technique using results from classical tolerance predictors. For completeness, we prove all three earlier results using conventional arguments in respectively \cref{app:prop1}, \cref{app:prop2} and \cref{app:prop_new}. Below, we show \cshref{prop:dist_exp_cov} contains \cshref{prop:marg_cov_split} and part \textit{(i)} of \cshref{prop:split_mc_tol}. Similar arguments apply to derive \cshref{prop:tol_reg_split} and part \textit{(ii)} of \cshref{prop:split_mc_tol} by first considering \cshref{rem:beta_P_hat}.

First, we show \cshref{prop:dist_exp_cov} contains \cshref{prop:marg_cov_split}. Note that the mean of $Z\sim\mathrm{Beta}(\lceil (1-\alpha)(n+1)\rceil,\lfloor \alpha(n+1)\rfloor)$ for a pre-specified $\alpha\in(0,1)$ satisfies
\begin{equation}\label{eq:beta_mean}
    1-\alpha\leq \mathbb{E}\big[Z\big]=1-\frac{\lfloor \alpha(n+1)\rfloor}{n+1}\leq 1-\alpha+\frac{1}{n+1}.
\end{equation}
Taking expectations in \cref{eq:dist_cov_exp_nonc} yields
\begin{equation}
    \mathbb{E}\Big[\mathbb{P}\big[Y_{n+1}\in \mathcal{S}_{\widehat{\lambda}_\mathrm{split}}(X_{n+1}) \ \big| \ \{(X_i,Y_i)\}^n_{i=1}\big]\Big]\geq \mathbb{E}\big[Z\big]\geq 1-\alpha,
\end{equation}
i.e. the marginal coverage of $\mathcal{S}_{\widehat{\lambda}_\mathrm{split}}(X_{n+1})$ is at least $1-\alpha$, obtaining the first part of \cshref{prop:marg_cov_split}. If the non-conformity scores are almost surely distinct, then $\mathcal{S}_{\widehat{\lambda}_\mathrm{split}}(X_{n+1})$ has marginal coverage satisfying the bounds in \cref{eq:beta_mean}, thus obtaining the second part of \cshref{prop:marg_cov_split}.

Second, we show \cshref{prop:dist_exp_cov} contains part \textit{(i)} of \cshref{prop:split_mc_tol}. From \cshref{prop:dist_exp_cov} directly follows that $\mathcal{S}_{\widehat{\lambda}_\mathrm{split}}(X_{n+1})$ with $\widehat{\lambda}_\mathrm{split}=\widehat{Q}(\alpha)$ is an $(\epsilon,\delta)$-tolerance region if
\begin{equation}\label{eq:tol_reg_beta}
    \mathrm{Beta}(1-\epsilon; \lceil (1-\alpha)(n+1)\rceil,\lfloor \alpha(n+1)\rfloor)\leq\delta,
\end{equation}
where we remind the reader that $\mathrm{Beta}(p; m, k)$ denotes the regularized incomplete Beta function with $p\in(0,1)$ and parameters $m,k>0$. This is equivalent to the condition on $\delta$ in part \textit{(i)} in \cshref{prop:split_mc_tol} by considering the well-known result from probability theory that for all $k, m\in\mathbb{N}$ where $k\leq m$, and all $p\in(0,1)$,
\begin{equation}
    \mathrm{Beta}(1-p;m+1-k,k)  = \mathrm{Bin}(k-1;m,p),
\end{equation}
referring to \cite[p. 22]{vovk2012conditional} for a simple proof. If the non-conformity scores are almost surely distinct, then \cshref{prop:dist_exp_cov} implies that $\mathcal{S}_{\widehat{\lambda}_\mathrm{split}}(X_{n+1})$ is an $(\epsilon,\delta)$-tolerance region if and only if \cref{eq:tol_reg_beta} holds, thus obtaining the last part of part \textit{(i)} of \cshref{prop:split_mc_tol}.

\subsection{Distribution of Empirical Coverage}\label{sec:dist_cov_fin_sample}

In practice, we use a finite test set to estimate $\mathbb{P}\big[Y_{n+1}\in \mathcal{S}_{\widehat{\lambda}_\mathrm{split}}(X_{n+1}) \ \big| \ \{(X_i,Y_i)\}^n_{i=1}\big]$, the probability of coverage of a split conformal prediction set, conditional on the calibration set. Say we have access to a finite iid test set $\{(X_i,Y_i)\}^{n+n_\mathrm{test}}_{i=n+1}$ of size $n_\mathrm{test}$. Consider $R$ random splits of the available $n+n_\mathrm{test}$ datapoints $\{(X_i,Y_i)\}^{n+n_\mathrm{test}}_{i=1}$ into a calibration and test set. For $j=1,\dots,R$, this yields a calibration set $\{(X_{i,j},Y_{i,j})\}^n_{i=1}$ and a test set $\{(X_{i,j},Y_{i,j})\}^{n+n_\mathrm{test}}_{i=n+1}$. Choosing $\widehat{\lambda}_{\mathrm{split}, j}$ based on the calibration set, we subsequently calculate the empirical coverage on the test set $C_j$ according to
\begin{equation}
    C_j:=\frac{1}{n_\mathrm{test}}\sum_{i=n+1}^{n+n_\mathrm{test}}\mathbbm{1}\big\{Y_{i,j}\in \mathcal{S}_{\widehat{\lambda}_{\mathrm{split}, j}}(X_{i,j}) \big\}, \quad j=1,\dots,R.
\end{equation}

Given the calibration set and iid data, $C_j$ is a consistent estimator for the probability of coverage, conditional on the calibration set, by the Law of Large Numbers. Its distribution is Binomial since $C_j$ is the average of indicator functions. Unconditional on the calibration set, the mean of the Binomial distribution is a random variable, following a Beta distribution by \cshref{prop:dist_exp_cov}. Therefore, the empirical coverage $C_j$ is a Beta-Binomial random variable, as summarized in \cshref{prop:dist_avg_cov}. As a general rule of thumb, a small number of test points $n_\mathrm{test}$ increases the variance of the empirical coverage $C_j$ and as such, the variance of $C_j$ will be larger than the Beta distribution specified in \cshref{prop:dist_exp_cov}. 

\begin{proposition}[Distribution of Empirical Coverage]\label{prop:dist_avg_cov}
    Let $\alpha\in(0,1)$, $j=1,\dots,R$ and suppose the sample $\{(X_i,Y_i)\}_{i=1}^{n+n_\mathrm{test}}$ is iid. Then the empirical coverage $C_j$ of the split conformal prediction set $\mathcal{S}_{\widehat{\lambda}_{\mathrm{split}, j}}(\cdot)$ (\cref{eq:split_set}) with $\widehat{\lambda}_{\mathrm{split}, j}=\widehat{Q}_j(\alpha)$ (\cref{eq:emp_quan}) satisfies
    \begin{equation}
        C_j:=\frac{1}{n_\mathrm{test}}\sum_{i=n+1}^{n+n_\mathrm{test}}\mathbbm{1}\big\{Y_{i,j}\in \mathcal{S}_{\widehat{\lambda}_{\mathrm{split}, j}}(X_{i,j}) \big\}\sim \frac{1}{n_\mathrm{test}}\mathrm{Bin}(n_\mathrm{test}, \mu),
    \end{equation}
    where
    \begin{equation}
        \mu\geq Z \quad \mathrm{and} \quad Z\sim\mathrm{Beta}(\lceil (1-\alpha)(n+1)\rceil,\lfloor \alpha(n+1)\rfloor),
    \end{equation}
    with equality if the non-conformity scores are almost surely distinct.

\end{proposition}
\begin{proof}
    See \cref{app:prop5}.
\end{proof}

\begin{remark}
    If the non-conformity scores are almost surely distinct, we write
    \begin{equation}
        C_j\sim\frac{1}{n_\mathrm{test}}\mathrm{BetaBin}(n_\mathrm{test},\lceil (1-\alpha)(n+1)\rceil,\lfloor \alpha(n+1)\rfloor),
    \end{equation}
    where $\mathrm{BetaBin}(\ell,m,k)$ is a Beta-Binomial distribution with $\ell,m,k\in\mathbb{N}$ positive integers.
\end{remark}

\begin{remark}\label{rem:C_bar}
    The statistic $\overline{C}=\frac{1}{R}\sum^R_{j=1}C_j$ is an estimate of the marginal coverage of $\mathcal{S}_{\widehat{\lambda}_\mathrm{split}}(\cdot)$. The distribution of $\overline{C}$ does not have a closed form solution, but since $\overline{C}$ is a random variable stochastically dominating the average of $R$ independent Beta-Binomial random variables, it satisfies
    \begin{equation}
    \begin{split}
        \mathbb{E}\big[\overline{C}\big]&=\mathbb{E}\big[C_j\big]\geq1-\frac{\lfloor\alpha(n+1)\rfloor}{n+1}, \\
        \mathbb{V}\big[\overline{C}\big]&=\frac{\mathbb{V}\big[C_j\big]}{R}\geq\frac{\lfloor\alpha(n+1)\rfloor\lceil (1-\alpha)(n+1)\rceil(n+n_\mathrm{test}+1)}{Rn_\mathrm{test}(n+1)^2(n+2)}=\mathcal{O}\Big(\frac{1}{R\min\{n, n_\mathrm{test}\}}\Big),
    \end{split}
    \end{equation}
    with equality if the non-conformity scores are almost surely distinct. Since $\mathbb{V}\big[\overline{C}\big]\rightarrow 0$ as $R\rightarrow\infty$, we have for large $R$ that $\overline{C}$ is approximately at least $1-\alpha$. Furthermore, if the non-conformity scores are almost surely distinct, then approximately $1-\alpha\leq\overline{C}\leq 1-\alpha+\frac{1}{n+1}$.
\end{remark}

\begin{remark}\label{rem:delta_bar}
    The statistic $\widehat{\delta}$ given by
    \begin{equation}
        \widehat{\delta}:=\frac{1}{R}\sum^R_{j=1}\mathbbm{1}\big\{C_j\leq\mathrm{Beta}(\epsilon;\lceil (1-\alpha)(n+1)\rceil,\lfloor \alpha(n+1)\rfloor)\big\}
    \end{equation}
    is a biased estimate of the parameter $\delta$ since the finite number of test points $n_\mathrm{test}$ increases the variance of the empirical coverage $C_j$. A better estimate for $\delta$ is
    \begin{equation}
        \overline{\delta}:=\frac{1}{R}\sum^R_{j=1}\mathbbm{1}\big\{C_j\leq \frac{1}{n_\mathrm{test}}\mathrm{BetaBin}(\epsilon;n_\mathrm{test},\lceil (1-\alpha)(n+1)\rceil,\lfloor \alpha(n+1)\rfloor)\big\}.
    \end{equation}
    Then approximately $\overline{\delta}\leq\delta$ in practice with equality if the non-conformity scores are almost surely distinct.
\end{remark}

\section{Distribution-Free Risk Control}\label{sec:riskcontrol}

This section shows how new procedures in distribution-free risk control contain split conformal prediction as a special case. \cref{sec:crc} discusses conformal risk control (CRC) \cite{angelopoulos2022conformal}, \cref{sec:ucb} considers upper confidence bound (UCB) calibration \cite{bates2021distribution} and finally \cref{sec:ltt} studies learn then test (LTT) \cite{angelopoulos2021learn}. The focus of this section is to get a better understanding of distribution-free risk control, and in particular its relation to split conformal prediction.

\subsection{Conformal Risk Control}\label{sec:crc}
\cite{angelopoulos2022conformal} extended split conformal prediction to conformal risk control (CRC), producing prediction sets that satisfy the MRC (\cshref{def:mrc}) finite-sample guarantee for a subset of risks (\cshref{def:crc}) under the similar assumption of iid data. This includes marginal coverage as a special case for risk functions as in \cref{eq:01loss}. Below, we provide a brief outline of CRC and how it produces identical prediction sets to split conformal prediction if the target guarantee is marginal coverage.

Similar to split conformal prediction, CRC starts with the sequence of nested prediction sets $\{\mathcal{S}_\lambda(X_{n+1})\}_{\lambda\in\Lambda}$, constructed from a given base predictor. The sequence is considered fixed and identical in both CRC and split conformal prediction. The key objective is $\widehat{\lambda}_\mathrm{CRC}$, the smallest value of $\lambda$ such that the prediction set $\mathcal{S}_{\widehat{\lambda}_\mathrm{CRC}}(X_{n+1})$ satisfies MRC with significance level $\alpha\in(-\infty,B]$, $B<+\infty$, under the assumption of an iid sample $\{(X_i,Y_i)\}^{n+1}_{i=1}$. 

CRC applies to bounded monotone risk functions $R:\Lambda\rightarrow(-\infty,B]$ of the form
\begin{equation}
    R(\lambda)=\mathbb{E}\big[L(Y_{n+1},\mathcal{S}_\lambda(X_{n+1}))\big],
\end{equation}
where $L:\mathcal{Y}\times\mathcal{Y}'\rightarrow(-\infty,B]$ is a bounded nested loss function for some $B<+\infty$. The loss function is nested (`monotone') in the sense of \cref{eq:mon_cond}. Note that the risk is a deterministic function of $\lambda$, but the risk corresponding to $\widehat{\lambda}_\mathrm{CRC}$ is dependent on the calibration set and thus random.

Consider the iid collection of non-increasing random functions $\ell_i:\Lambda\rightarrow(-\infty,B]$ for $i=1,\dots,n+1$, where $\ell_i(\lambda):=L(Y_i,\mathcal{S}_\lambda(X_i))$. Furthermore, write the empirical risk function on the calibration set as $\widehat{R}(\lambda):= \frac{1}{n}\sum^n_{i=1}\ell_i(\lambda)$, i.e. the average loss of a prediction set $\mathcal{S}_\lambda(\cdot)$ on the calibration set. Then \cite{angelopoulos2022conformal} shows that given the calibration set, $\widehat{\lambda}_\mathrm{CRC}$ is
\begin{equation}\label{eq:lam_crc}
    \widehat{\lambda}_\mathrm{CRC} :=\inf\big\{\lambda\in\Lambda : \frac{n}{n+1}\widehat{R}(\lambda)+\frac{B}{n+1}\leq\alpha\big\},
\end{equation}
such that $\mathcal{S}_{\widehat{\lambda}_\mathrm{CRC}}(X_{n+1})$ satisfies MRC with significance level $\alpha\in(-\infty,B]$. Furthermore, if the $\ell_i$ are iid from a continuous distribution, then the expected loss is not too conservative, i.e.
\begin{equation}\label{eq:lb_crc}
    \alpha-\frac{2B}{n+1}\leq\mathbb{E}\big[R(\widehat{\lambda}_\mathrm{CRC}) \big]\leq \alpha,
\end{equation}

In split conformal prediction, we are interested in the conditional probability that a new label $Y_{n+1}$ is contained in the prediction set $\mathcal{S}_\lambda(X_{n+1})$. The corresponding risk is the expectation of the 0-1 loss, i.e.
\begin{equation}\label{eq:R_split}
    \begin{split}
        R_\mathrm{0-1}(\lambda)&=\mathbb{E}\Big[\mathbbm{1}\big\{Y_{n+1}\not\in\mathcal{S}_\lambda(X_{n+1})\big\}\Big] \\
        &=\mathbb{E}\Big[\mathbbm{1}\big\{r(X_{n+1},Y_{n+1})>\lambda\big\}\Big], 
    \end{split}
\end{equation}
where the last equality involving the random score function $r(X,Y)$ (\cref{eq:score}) follows from the fact that $Y_{n+1}\in \mathcal{S}_\lambda(X_{n+1})$ if and only if $r(X_{n+1},Y_{n+1})$ is at most $\lambda$ (\cref{eq:split_set}).
A higher risk indicates a higher probability that $Y_{n+1}$ is not contained in $\mathcal{S}_\lambda(X_{n+1})$. 

The 0-1 loss satisfies the above monotonicity requirement and is bounded by $B=1$. Furthermore, the MRC guarantee reduces to marginal coverage at significance level $\alpha\in(0,1)$. The empirical risk $\widehat{R}_{0-1}(\lambda)$ is the miscoverage rate of a prediction set $\mathcal{S}_{\lambda}(X_{n+1})$ on the calibration set. Now we can write \cref{eq:lam_crc} for risks as in \cref{eq:R_split} as
\begin{equation}
\begin{split}
    \widehat{\lambda}_{\mathrm{CRC}} &=\inf\big\{\lambda\in\Lambda : \frac{1}{n+1}\sum^n_{i=1}\mathbbm{1}\big\{r(X_{i},Y_{i})>\lambda\big\}+\frac{1}{n+1}\leq\alpha\big\}\\
    &=\inf\big\{\lambda\in\Lambda : \sum^n_{i=1}\mathbbm{1}\big\{r(X_{i},Y_{i})\leq\lambda\big\}\geq(1-\alpha)(n+1)\big\}\\
    &=\inf\big\{\lambda\in\Lambda : \frac{1}{n}\sum^n_{i=1}\mathbbm{1}\big\{r(X_{i},Y_{i})\leq\lambda\big\}\geq\frac{1}{n}\lceil(1-\alpha)(n+1)\rceil\big\}\\
    &=\widehat{Q}(\alpha).
\end{split}
\end{equation}
This is identical to the value of $\lambda_\mathrm{split}$ specified in \cref{eq:emp_quan} which results in marginal coverage. Informally, CRC can be summarized as below. That is, for risks as in \cref{eq:R_split} using the 0-1 loss, CRC obtains $\widehat{\lambda}_\mathrm{CRC}=\widehat{Q}(\alpha)$ such that the corresponding prediction set satisfies marginal coverage at significance level $\alpha$. Split conformal prediction simply relates to the bottom line.
\begin{equation*}
\begin{split}
    \mathrm{CRC}:\quad  \mathrm{bounded} \ \mathrm{monotone}\ \mathrm{risks}\quad&\longrightarrow\quad \widehat{\lambda}_\mathrm{CRC} \quad\longrightarrow\quad\mathrm{MRC}\\
    0-1\ \mathrm{loss}\quad &\longrightarrow\quad \widehat{Q}(\alpha)\quad \longrightarrow\quad\mathrm{marginal}\ \mathrm{coverage}
\end{split}
\end{equation*}
Finally, if the random functions $\ell_i$ are iid from a continuous distribution, the lower bound in \cref{eq:lb_crc} holds. \cite{angelopoulos2022conformal} shows that the $\ell_i$ are iid if and only if the random score function $r(X,Y)$ follows a continuous distribution. The resulting lower bound in \cref{eq:lb_crc} is not as tight as the lower bound presented in \cshref{prop:marg_cov_split}. This is an inconsistency of CRC theory with the known results for split conformal prediction and an interesting direction for future research.

\subsection{Upper Confidence Bound Calibration}\label{sec:ucb}
\cite{bates2021distribution} proposed a procedure called upper confidence bound (UCB) calibration to calibrate prediction sets to obtain an $(\epsilon,\delta)$-RCP (\cshref{def:risk_control}) for a subset of risks (\cshref{rem:rcps}) under the similar assumption of iid data. This includes the construction of an $(\epsilon,\delta)$-tolerance region as a special case for risk functions as in \cref{eq:01loss}. In this section, we describe the procedure and show that the tolerance regions obtained through UCB calibration are identical to those obtained through split conformal prediction. 

Similar to split conformal prediction, UCB calibration starts with the sequence of nested prediction sets $\{\mathcal{S}_\lambda(X_{n+1})\}_{\lambda\in\Lambda}$. The objective of interest is $\widehat{\lambda}_\mathrm{UCB}$, the smallest value of $\lambda$ such that the prediction set $\mathcal{S}_{\widehat{\lambda}_\mathrm{UCB}}(X_{n+1})$ is an $(\epsilon,\delta)$-RCP. UCB calibration applies to any monotone risk function $R:\Lambda\rightarrow\mathbb{R}$. However, to make the connection to split conformal prediction clear, we define the risk function as the expectation of a continuous loss function, i.e.
\begin{equation}
    R(\lambda)=\mathbb{E}\big[L(Y_{n+1},\mathcal{S}_\lambda(X_{n+1}))\big],
\end{equation}
where $L:\mathcal{Y}\times\mathcal{Y}'\rightarrow\mathbb{R}_{\geq0}$ is a nested non-negative loss function. The loss function is nested (`monotone') in the sense of \cref{eq:mon_cond}. Again, note that the risk is a deterministic function of $\lambda$, but the risk corresponding to $\widehat{\lambda}_\mathrm{UCB}$ is dependent on the calibration set and thus random.

A key assumption in UCB calibration is that we have access to a pointwise data-dependent UCB $\widehat{R}^+(\lambda)$ for the risk $R(\lambda)$, such that for a pre-specified $\delta\in(0,1)$,
\begin{equation}\label{eq:pw_ucb}
    \mathbb{P}\big[R(\lambda)\leq \widehat{R}^+(\lambda)\big]\geq 1-\delta, \quad \forall\lambda\in\Lambda.
\end{equation}
In practice, we do not have access to the risk function $R(\lambda)$ and instead construct the UCB based on the empirical risk $\widehat{R}(\lambda):=\frac{1}{n}\sum^n_{i=1}\ell_i(\lambda)$, where the $\ell_i$ are defined as in \cref{sec:crc}. The empirical risk is the average loss of a prediction set $\mathcal{S}_\lambda(\cdot)$ on the calibration set.

By searching over the space $\Lambda$, \cite{bates2021distribution} shows that the target $\widehat{\lambda}_\mathrm{UCB}$ satisfies
\begin{equation}\label{eq:lam_ucb}
    \widehat{\lambda}_\mathrm{UCB}:=\inf\big\{\lambda\in\Lambda :\epsilon\geq \widehat{R}^+(\lambda'),  \forall \lambda'\geq\lambda\big\},
\end{equation}
i.e. the smallest value of $\lambda$ such that the risk function $R(\lambda')$ is upper bounded by $\epsilon$ for all $\lambda'\geq\widehat{\lambda}_\mathrm{UCB}$. In other words, that the risk $R(\widehat{\lambda}_\mathrm{UCB})$ is at most $\epsilon$ with probability at least $1-\delta$ and thus that $\mathcal{S}_{\widehat{\lambda}_\mathrm{UCB}}(X_{n+1})$ an $(\epsilon,\delta)$-RCP.

In split conformal prediction, the risk function of interest is $R_{0-1}(\lambda)$ as defined in \cref{eq:R_split}. This risk satisfies the above monotonicity requirement. Furthermore, for this risk the $(\epsilon,\delta)$-RCP guarantee reduces to an $(\epsilon,\delta)$-tolerance region. Now the empirical risk on the calibration set becomes
\begin{equation}
    \widehat{R}_{0-1}(\lambda) = \frac{1}{n}\sum^n_{i=1}\mathbbm{1}\big\{r(X_i,Y_i)>\lambda\big\}.
\end{equation}
\cite{bates2021distribution} introduced the Hoeffding-Bentkus (HB) UCB for bounded loss functions, which is a combination of Hoeffding's inequality \cite{hoeffding1994probability} and Bentkus' inequality \cite{bentkus2004hoeffding}. In the case of the 0-1 loss, the Bentkus bound improved by a factor $e$ is the most precise and should always be used. The intuition behind this is that for a fixed $\lambda$ and an iid sample, each $\ell_i$ is an independent Bernoulli trial such that the empirical risk $\widehat{R}(\lambda)$ is the mean of a Binomial random variable. The Bentkus bound for the empirical risk $\widehat{R}_{0-1}(\lambda)$ is given by
\begin{equation}\label{eq:ucb_bin}
    \widehat{R}^+_{0-1}(\lambda)=\inf\big\{p : \mathrm{Bin}\big( n\widehat{R}_{0-1}(\lambda); n, p\big)\leq\delta\big\}.
\end{equation}
Substituting $\widehat{R}^+_{0-1}(\lambda)$ in the definition of $\widehat{\lambda}_\mathrm{UCB}$ yields
\begin{equation}
\begin{split}
    \widehat{\lambda}_\mathrm{UCB}&=\inf\Big\{\lambda\in\Lambda : \epsilon\geq\inf\big\{p : \mathrm{Bin}\big( n\widehat{R}_{0-1}(\lambda); n, p\big)\leq\delta\big\}\Big\}\\
    &=\inf\Big\{\lambda\in\Lambda : \widehat{R}_{0-1}(\lambda)\leq\frac{1}{n}\sup\big\{k : \mathrm{Bin}\big( k; n, \epsilon\big)\leq\delta\big\}\Big\}\\
    &=\Big(1-\frac{1}{n}\sup\big\{k:\mathrm{Bin}(k;n,\epsilon)\leq\delta\big\}\Big)\text{-th quantile of }\{r_i\}^n_{i=1}\\
    &=\widehat{P}(\epsilon,\delta),
\end{split}
\end{equation}
where we inverted the Binomial distribution for $\widehat{R}_{0-1}(\lambda)$ in the second step and used the definition of the non-conformity score function (\cref{eq:score}) in the final step. This is identical to the value of $\lambda_\mathrm{split}$ that result in an $(\epsilon,\delta)$-tolerance region, see \cshref{prop:tol_reg_split}. Informally, UCB calibration can be summarized as below. That is, for risks as in \cref{eq:R_split} using the 0-1 loss, UCB calibration obtains $\widehat{\lambda}_\mathrm{UCB}=\widehat{P}(\epsilon,\delta)$ such that the corresponding prediction set is an $(\epsilon,\delta)$-tolerance region. Split conformal prediction simply relates to the bottom line.
\begin{equation*}
\begin{split}
    \mathrm{UCB}:\quad \mathrm{monotone}\ \mathrm{risks}\quad&\longrightarrow\quad \widehat{\lambda}_\mathrm{UCB}\ \ \ \quad\longrightarrow\quad(\epsilon,\delta)-\mathrm{RCP}\\
    0-1\ \mathrm{loss}\quad &\longrightarrow\quad \widehat{P}(\epsilon,\delta)\quad \longrightarrow\quad(\epsilon,\delta)-\mathrm{tolerance}\ \mathrm{region}
\end{split}
\end{equation*}

\subsection{Learn Then Test}\label{sec:ltt}
\cite{angelopoulos2021learn} generalized UCB calibration to allow for a more general notion of statistical error than only monotonic risk functions, coined learn then test (LTT). To this extend, the problem of risk control is recast as a multiple testing problem, linking distribution-free predictive inference to the rich literature on multiple hypothesis testing. Here, we provide a very brief introduction and highlight the connection to UCB calibration and thus split conformal prediction. 

Although the LTT framework is designed to be able to control multiple risks simultaneously, for the purposes of this thesis we restrict ourselves to a one-dimensional risk. The goal of LTT is to find a set $\widehat{\Lambda}_\mathrm{LTT}\subseteq\Lambda_\mathrm{LTT}$ such that for all $\lambda\in\widehat{\Lambda}_\mathrm{LTT}$, $\mathcal{S}_\lambda(X_{n+1})$ is an $(\epsilon,\delta)$-RCP. An important difference with UCB calibration is that the LTT procedure requires a discrete grid $\Lambda_\mathrm{LTT}=\{\lambda_1,\dots,\lambda_N\}$, such that if $\Lambda_\mathrm{LTT}$ is not already discrete, it must be discretized.

First, each $\lambda_j\in\Lambda_\mathrm{LTT}=\big\{\lambda_1,\dots,\lambda_N\big\}$ is associated with a null hypothesis
\begin{equation}
    \mathcal{H}_j:R(\lambda_j)>\epsilon.
\end{equation}
Rejecting $\mathcal{H}_j$ corresponds to selecting $\lambda_j$ as a point such that $\mathcal{S}_{\lambda_j}(X_{n+1})$ is an $(\epsilon,\delta)$-RCP. For each $\mathcal{H}_j$, a finite-sample valid $p$-value is calculated using an UCB, for example the HB bound for bounded losses introduced in \cite{bates2021distribution}. $p_j$ is a valid $p$-value for $\mathcal{H}_j$ if and only if $p_j$ has a distribution that stochastically dominates the standard uniform distribution $\mathcal{U}[0,1]$, i.e. for all $u\in[0,1]$, $\mathbb{P}\big[p_j\leq u\big]\leq u$. In the case of split conformal prediction, we can use the exact $p$-values corresponding to the binomial distribution, e.g. \cref{eq:ucb_bin}. 

Many nearby $p$-values are often highly dependent, since the same calibration set is used to estimate them. The challenge is to reject as many $\mathcal{H}_j$ as possible, while limiting the false positive rate by a pre-specified level $\delta\in(0,1)$. To do so, the familywise-error rate (FWER) \cite{bauer1991multiple} needs to be controlled. An FWER controlling algorithm is such that even for the largest value in $\widehat{\Lambda}_\mathrm{LTT}$, the corresponding prediction set $\mathcal{S}_{\sup\lambda\in\widehat{\Lambda}_\mathrm{LTT}}(X_{n+1})$ is an $(\epsilon,\delta)$-RCP. The simplest example is the Bonferroni correction \cite{bonferroni1936teoria, holm1979simple}, yielding $\widehat{\Lambda}_\mathrm{LTT}=\big\{\lambda_j: p_j<\frac{\delta}{|\Lambda_\mathrm{LTT}|}\big\}$.

A refinement of the Bonferroni correction is called fixed sequence testing \cite{bauer1991multiple, angelopoulos2021learn}, leveraging the dependency of $p$-values for nearby $\lambda$ and the corresponding smoothness with $\lambda$. In other words, oftentimes hypotheses are rejected in clusters of nearby $\lambda$. `Fixed sequence' refers to the prior ordering of $\lambda$ roughly from most likely to least likely to be rejected. Testing the sequence using $p_j\leq\delta$ until the first rejection yields a set of all hypotheses rejected so far. \cite{angelopoulos2021learn} shows that such an algorithm is guaranteed to control the FWER. The intuition behind the improvement compared to the Bonferroni correction is that the goal is not to reject as many hypotheses as possible, therefore limiting the multiplicity correction. 

In the case of split conformal prediction, we have a monotonic non-increasing risk function as in \cref{eq:R_split}. Using a closed discrete set $\Lambda_\mathrm{LTT}\subset\mathbb{R}\cup\{\pm\}$, we can use fixed sequence testing with the reverse natural ordering to obtain $\widehat{\Lambda}_\mathrm{LTT}$, starting at $\sup\Lambda_\mathrm{LTT}$ and proceeding to $\inf\Lambda_\mathrm{LTT}$. In combination with the exact $p$-values in \cref{eq:ucb_bin}, this results in essentially the same procedure as UCB calibration. One minor difference is that UCB calibration allows for a continuous space $\Lambda$, such that $\widehat{\Lambda}_\mathrm{LTT}$ might not contain $\widehat{\lambda}_\mathrm{UCB}$, although this difference is negligible for a finely grained grid.

\section{Experiments}\label{sec:experiments}

This section empirically validates the theoretical results in \cref{sec:conformal} and \cref{sec:dist_cov} in the context of regression, using conformalized quantile regression (CQR) \cite{romano2019conformalized} in combination with a quantile random forest (QRF) \cite{meinshausen2006quantile} base predictor. \cref{sec:methods} outlines the procedures and hyper-parameter tuning. Subsequently, \cref{sec:syn_data} benchmarks the tolerance regions produced by CQR against QRF without conformalization on a synthetic dataset. Then, \cref{sec:real_data} shows CQR correctly calibrates \textit{any} QRF base predictor on synthetic data as well as 4 real datasets. For all experiments mentioned in this thesis, we make extensive use of the codebase\footnote{\url{https://github.com/yromano/cqr}} provided by \cite{romano2019conformalized}.

\subsection{Conformalized Quantile Random Forests}\label{sec:methods}
We refer the reader back to \cref{sec:cqr} for an outline of CQR, which is conveniently implemented in the package-style codebase provided with \cite{romano2019conformalized}. In the experiments below we use a QRF base predictor, implemented in the package \texttt{sklearn}. The hyper-parameters of the QRF base predictor are the package defaults, except that we set the number of trees to 100 and the minimum number of leaves required at a node to 40. The quantile crossing problem cannot affect QRFs \cite{romano2019conformalized}.

\cite{romano2019conformalized} showed that this combination of base predictor and split conformal procedure is one of the top performing methods among standard and locally adaptive versions of split conformal prediction. However, we would like to remind the reader that the focus of this thesis is on prediction set validity, not efficiency. The prediction set size is highly dependent on the accuracy of the base predictor, but for computational reasons we operate a relatively small random forest. Therefore, the resulting prediction intervals might not be optimal in terms of size. Using a different base predictor or a thorough tuning process of hyper-parameters might result in smaller intervals, while preserving the finite-sample coverage guarantee. 

Following \cite{romano2019conformalized}, we tune the nominal quantiles of the QRF base predictor using cross-validation. In other words, we select the nominal quantiles that result in the shortest prediction intervals after calibration averaged over 10 folds of the proper training set, such that the target coverage of the base predictor is not exactly $1-\alpha$. This is because quantile regression is often too conservative, resulting in unnecessarily wide prediction intervals. Tuning the nominal quantiles of the base predictor does not invalidate the finite-sample coverage guarantee, but it might result in shorter intervals. 

For the sake of comparison, we include a QRF without conformalization as a baseline predictor in \cref{sec:syn_data}. Since QRF does not require a calibration set, we use the calibration set as additional training data. To ensure a fair comparison, the hyper-parameters are identical to the QRF base predictor in CQR, except that the nominal quantiles are fixed to $q_{\alpha/2}$ and $q_{1-\alpha/2}$. QRF without conformalization has asymptotic coverage of $1-\alpha$, but does not have a finite-sample coverage guarantee. 

Throughout this thesis, the target tolerance region is fixed at $\epsilon=\delta=0.1$. The full training set is consistently split into a proper training and calibration set of equal size. Given the size of the calibration set $n$, we obtain a target significance level $\alpha$ for marginal coverage through \cshref{prop:split_mc_tol}. The focus of the experiments lies on demonstrating prediction set validity, i.e. showing that the fitted prediction sets are $(\epsilon,\delta)$-tolerance regions and satisfy marginal coverage at significance level $\alpha$. To this extend, we use the statistics $\overline{C}$ (\cshref{rem:C_bar}) and $\overline{\delta}$ (\cshref{rem:delta_bar}) as estimators of the parameters $1-\alpha$ and $\delta$. We use $R=1000$ trials in all experiments.

\subsection{Synthetic Example}\label{sec:syn_data}

We start with simulated heteroskedastic data with outliers, identical to the synthetic example in \cite{romano2019conformalized}. Although \cite{romano2019conformalized} presented this example to illustrate the importance of adaptivity in conformal prediction, we focus on demonstrating prediction set validity in the context of tolerance regions. 

The data consists of iid uni-variate features $X_i$ from a uniform distribution $\mathcal{U}[1,5]$. The corresponding labels $Y_i$ are sampled as
\begin{equation}
    Y_i\sim \mathrm{Pois}(\sin^2(X_i)+0.1)+0.03X_i\gamma_{1,i}+25\mathbbm{1}\{U_i<0.01\}\gamma_{2,i},
\end{equation}
where the $\gamma_{1,i},\gamma_{2,i}\sim\mathcal{N}(0,1)$ are iid standard normal noise and the $U_i\sim\mathcal{U}[0,1]$ are iid standard uniform. Note that the last term creates a few but large outliers. 

We generate a proper training set of size $n_\mathrm{train}=1000$ and a calibration set of size $n=1000$. Furthermore, we generate a test set of size $n_\mathrm{test}=5000$. Given $\epsilon=\delta=0.1$ and the size of the calibration set $n=1000$, \cshref{prop:split_mc_tol} yields that a CQR prediction set with $\lambda_\mathrm{CQR}=0.913\%$-th quantile of the set of non-conformity scores $\{r_i\}^n_{i=1}$ yields an $(0.1,0.1)$-tolerance region. Furthermore, it has marginal coverage of $91.21\%$ ($\alpha=0.0879$). 

\cref{fig:test_data_outliers} visualises the test set on the full range of the label population $\mathcal{Y}$ and \cref{fig:cqr_rf_pred} shows the test set on a `zoomed in' subsection of $\mathcal{Y}$. We observe clear heteroskedasticity and several large outliers. Furthermore, \cref{fig:cqr_rf_pred} shows one instance of the CQR prediction interval and corresponding QRF predicted quantiles. To provide some intuition, the absolute difference between one of the predicted quantiles and the CQR prediction set equals $\lambda_\mathrm{CQR}$. The QRF base predictor without calibration achieves empirical coverage of 83.98\% on the test set. Note that this is intentionally lower than $1-\alpha$, since the nominal quantiles of the base predictor are tuned using cross-validation. Interestingly, the algorithm selects nominal quantiles below the nominal level, because the QRF base predictor seems to be too conservative. The CQR prediction interval achieves empirical coverage of 92.26\%, well above the target level, with an average interval length of 2.09. Notice in \cref{fig:cqr_rf_pred} how the length of constructed interval varies with $X$, reflecting adaptivity of the CQR prediction interval.

\begin{figure}[h]
    \centering
    \begin{subfigure}[b]{0.45\linewidth}
        \centering
        \includegraphics[width=\linewidth]{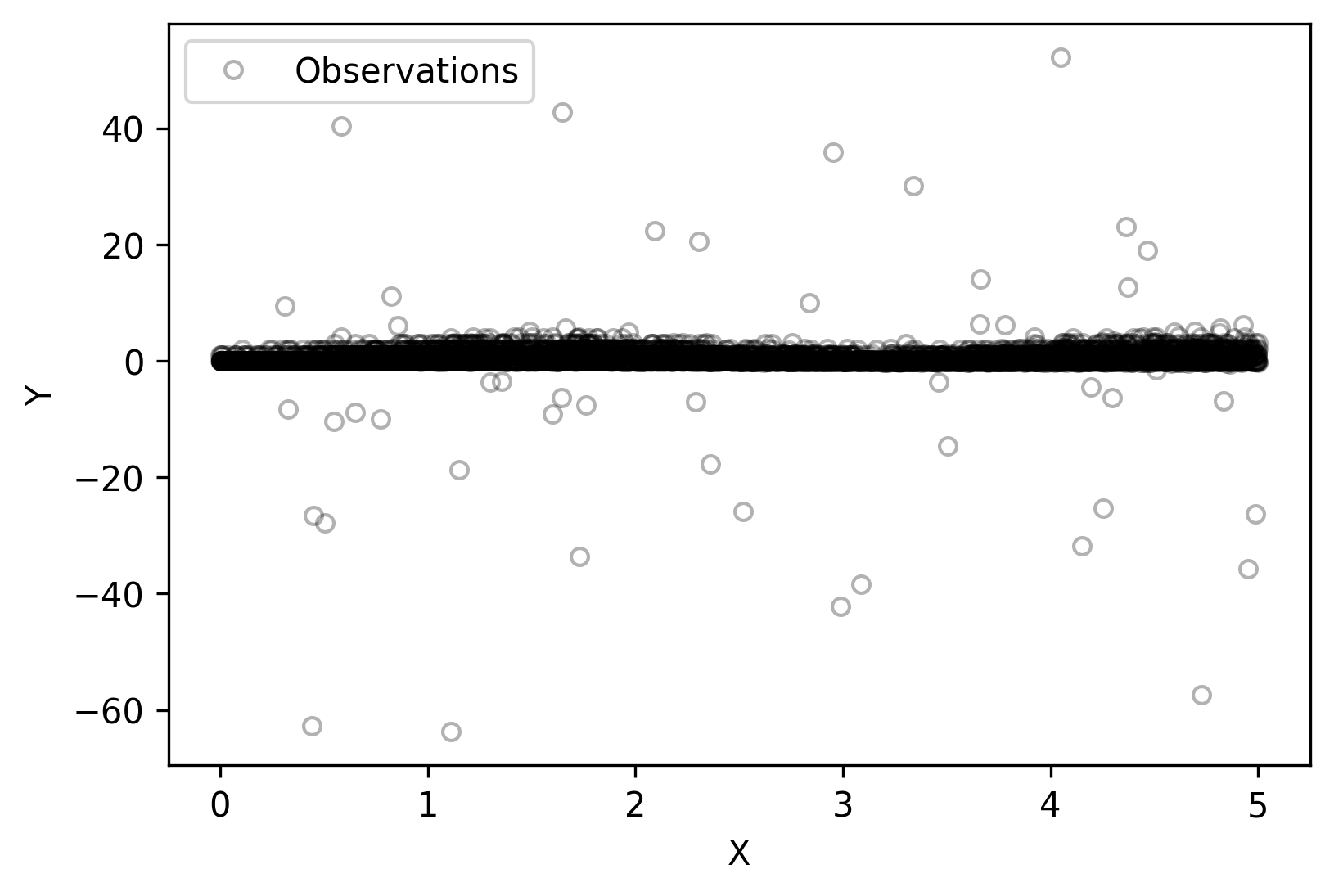}
        \caption{Full range scatterplot of test data.}
        \label{fig:test_data_outliers}
    \end{subfigure}
    \begin{subfigure}[b]{0.45\linewidth}
        \centering
        \includegraphics[width=\linewidth]{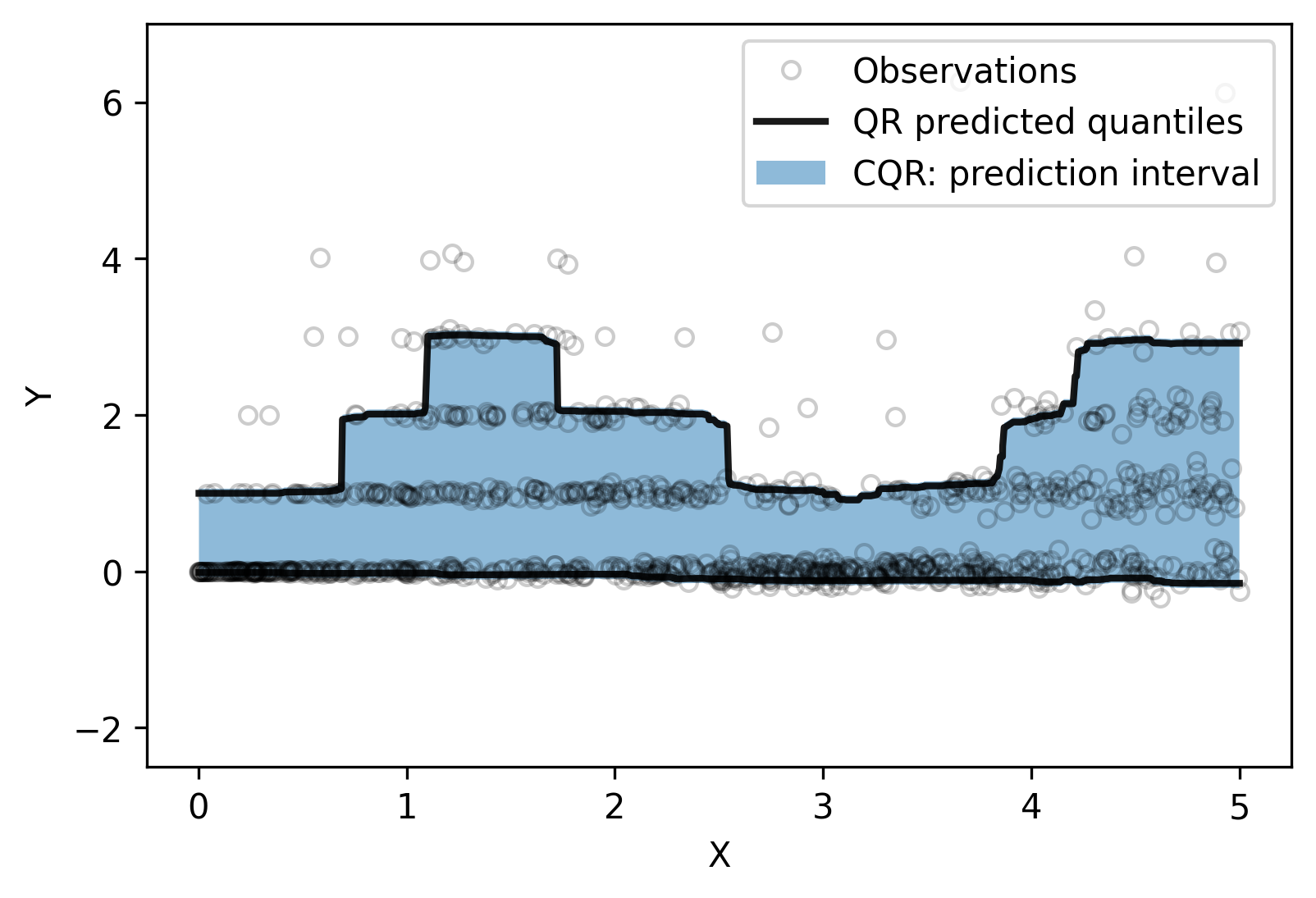}
        \caption{CQR Prediction Interval}
        \label{fig:cqr_rf_pred}
    \end{subfigure}
    \caption[Test data visualization and CQR prediction interval.]{Test data visualization and CQR prediction interval.}
    \label{fig:test_data}
\end{figure}

To validate CQR does in fact produce the correct $(\epsilon,\delta)$-tolerance region, we repeat the above procedure for $R=1000$ trials. That is, we generate a proper training, calibration and test set $R$ times and construct prediction sets using CQR. For each trial, we record the empirical coverage of the CQR prediction set on the test set and the average length of the CQR prediction interval. For the sake of comparison, we fit a QRF without conformalization on a combined proper training and calibration set and similarly record the empirical coverage and average interval length. 

\cref{fig:den_cov} shows a histogram of the empirical coverage on the test set after $R$ trials. In addition, \cref{fig:den_cov} contains a $\mathrm{Beta}(\lceil (1-\alpha)(n+1)\rceil,\lfloor \alpha(n+1)\rfloor)$ and $\mathrm{BetaBin}(n_\mathrm{test},\lceil (1-\alpha)(n+1)\rceil,\lfloor \alpha(n+1)\rfloor)$ distribution. For $n_\mathrm{test}=5000$, the Beta-Binomial distribution is very close to the Beta distribution. We observe that the CQR empirical coverage coincides with the Beta-Binomial distribution, as the theory in \cref{prop:dist_avg_cov} suggests. We observe that QRF without conformalization tends to undercover. This is unsurprising, since QRF does not have a finite-sample coverage guarantee, and demonstrates the importance of the post-hoc calibration. 

\cref{fig:hist_length} shows a histogram of the average interval lengths after $R$ trials. We observe that generally, CQR yields shorter interval lengths than QRF. This is likely due to the tuning of the nominal quantiles of the base predictor in CQR, which might result in shorter prediction intervals. In addition, the base predictor occasionally overcovers, which is mitigated by the signed conformity scores of \cref{eq:cqr_score}. 

\begin{figure}[h]
    \centering
    \begin{subfigure}[b]{0.45\linewidth}
        \centering
        \includegraphics[width=\linewidth]{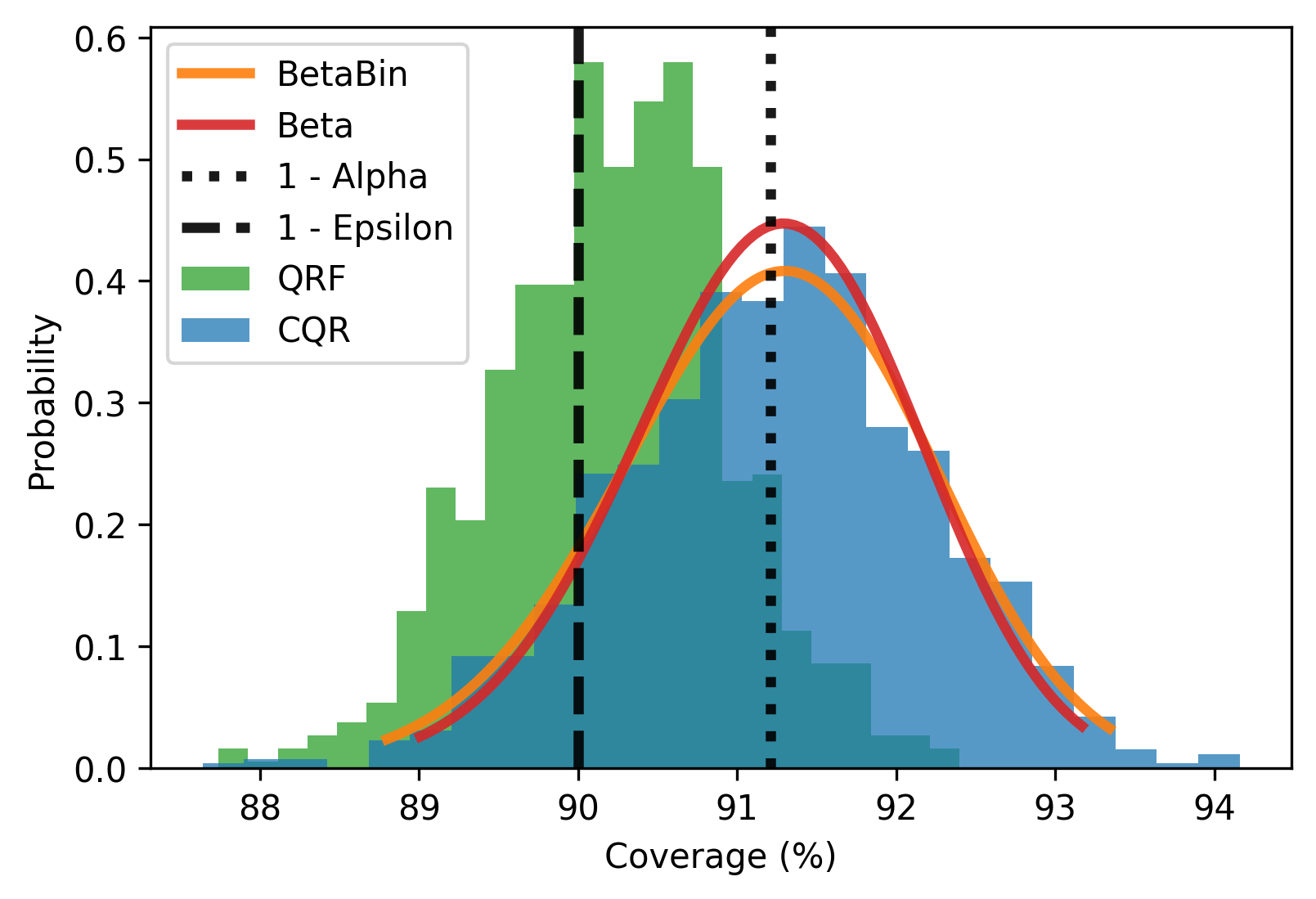}
        \caption{Histogram of empirical coverage.}
        \label{fig:den_cov}
    \end{subfigure}
    \begin{subfigure}[b]{0.45\linewidth}
        \centering
        \includegraphics[width=\linewidth]{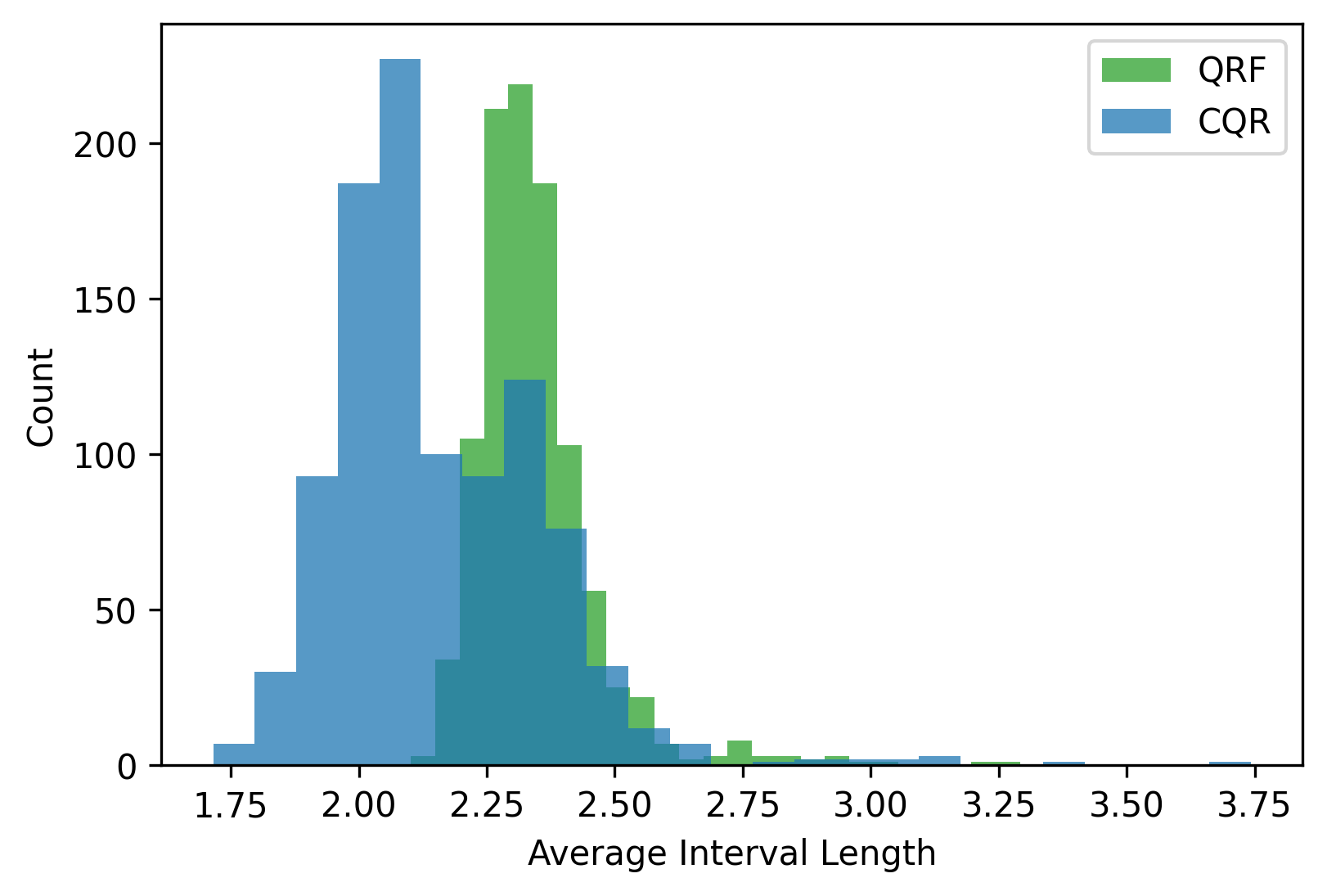}
        \caption{Histogram of average interval length.}
        \label{fig:hist_length}
    \end{subfigure}
    \caption[CQR versus QRF without conformalization.]{CQR versus QRF without conformalization. Number of trials $R=1000$. The $\mathrm{Beta}(\lceil (1-\alpha)(n+1)\rceil,\lfloor \alpha(n+1)\rfloor)$ and $\mathrm{BetaBin}(n_\mathrm{test},\lceil (1-\alpha)(n+1)\rceil,\lfloor \alpha(n+1)\rfloor)$ distribution are plotted on the left. More details in \cref{tab:syn_cqr_qrf}.}
    \label{fig:syn_cqr_qrf}
\end{figure}

The above observations are confirmed by the statistics displayed in \cref{tab:syn_cqr_qrf}. The average empirical coverage $\overline{C}$ after $R$ trials is very close to the target $1-\alpha$ for CQR, while QRF without conformalization undercovers. As expected, $\overline{\delta}$ is close to 10\% for CQR, while QRF is unsurprisingly not able to control $\delta$. The average interval length is a bit more than 8\% shorter in CQR than QRF.

\begin{table}[h]
    \setlength\tabcolsep{5pt}
    \centering
    \begin{tabular}{ccccccccc}
        \toprule
        Method & $n_\mathrm{train}$ & $n$ & $n_\mathrm{val}$ & $1-\alpha$ (\%) & $\overline{C}$ (\%) & $\delta$ (\%) & $\overline{\delta}$ (\%) & Avg. Length \\
        \midrule
        \rowcolor{lavender}
        CQR & 1000 & 1000 & 5000 & 91.21 & 91.22 & 10.00 & 9.70 & 2.15  \\
        QRF & 2000 & 0 & 5000 & 91.21 & 90.23 & 10.00 & 33.20 & 2.34  \\
        \bottomrule
    \end{tabular}
    \caption[Synthetic example benchmark of CQR.]{Statistics corresponding to \cref{fig:syn_cqr_qrf} for the experiment on synthetic data of CQR versus QRF without conformalization. $R=1000$ trials. See \cshref{rem:C_bar} and \cshref{rem:delta_bar} for the definition of $\overline{C}$ and $\overline{\delta}$.}
    \label{tab:syn_cqr_qrf}
\end{table}

An important detail in the setup of the above experiments is that for each of $R$ trials, a new base predictor is fit on the proper training set. As such, \cref{tab:syn_cqr_qrf} demonstrates that the CQR prediction sets are $(\epsilon,\delta)$-tolerance regions on average over the proper training set, i.e. \textit{unconditional} on the base predictor. However, the theoretical results in this thesis hold conditional on the proper training set and thus for \textit{any} base predictor. In the following section we demonstrate on synthetic and real data that CQR produces valid prediction sets for any base predictor, not only on average. 

\subsection{Calibrating a Given Base Predictor}\label{sec:real_data}

We construct tolerance regions by conformalizing $10$ QRF base predictors on synthetic data as well as 4 popular benchmark datasets for regression, namely bike-sharing (\texttt{bike}) \cite{bike}, Tennessee's student teacher achievement ratio (\texttt{star}) \cite{star}, community and crimes (\texttt{community}) \cite{community} and concrete compressive strength (\texttt{concrete}) \cite{concrete}. Note that this is a subset of the 11 benchmark datasets used in \cite{romano2019conformalized}, excluding large datasets for computational reasons. We do not benchmark against QRF without conformalization, since the focus here is to demonstrate post-hoc calibration for any base predictor using CQR. 

Throughout the 10 runs, we randomly reserve 40\% of the available data for the proper training set to train a QRF base predictor. The remaining data is randomly split $R=1000$ times into a calibration set of size equal to 40\% of the total data and a test set equal to 20\% of the total data. The features are standardized to have zero mean and unit variance using sample means and variance computed on the proper training set. The labels are rescaled by dividing by the mean absolute value of the labels in the proper training set. For each trial, the same base predictor is calibrated on a new calibration set and subsequently validated on the test set. This makes a total of 10 runs of $R=1000$ trials on 5 datasets for a total of 50000 experiments.

\cref{fig:real_data} shows histograms of empirical coverage and average interval lengths of $R$ trials for one base predictor. We observe that for each dataset, the histogram of empirical coverage coincides with the plotted $\mathrm{BetaBin}(n_\mathrm{test},\lceil (1-\alpha)(n+1)\rceil,\lfloor \alpha(n+1)\rfloor)$ distribution. For small values of $n_\mathrm{test}$, the $\mathrm{BetaBin}$ distribution is clearly distinct from the plotted $\mathrm{Beta}(\lceil (1-\alpha)(n+1)\rceil,\lfloor \alpha(n+1)\rfloor)$ distribution. The histograms confirm that CQR correctly controls the empirical coverage on each dataset, even in the case of real data where the assumption of iid data is never fully satisfied. 

\begin{figure}[h!]
    \centering
    
    \begin{subfigure}[b]{0.45\linewidth}
        \centering
        \includegraphics[width=\linewidth]{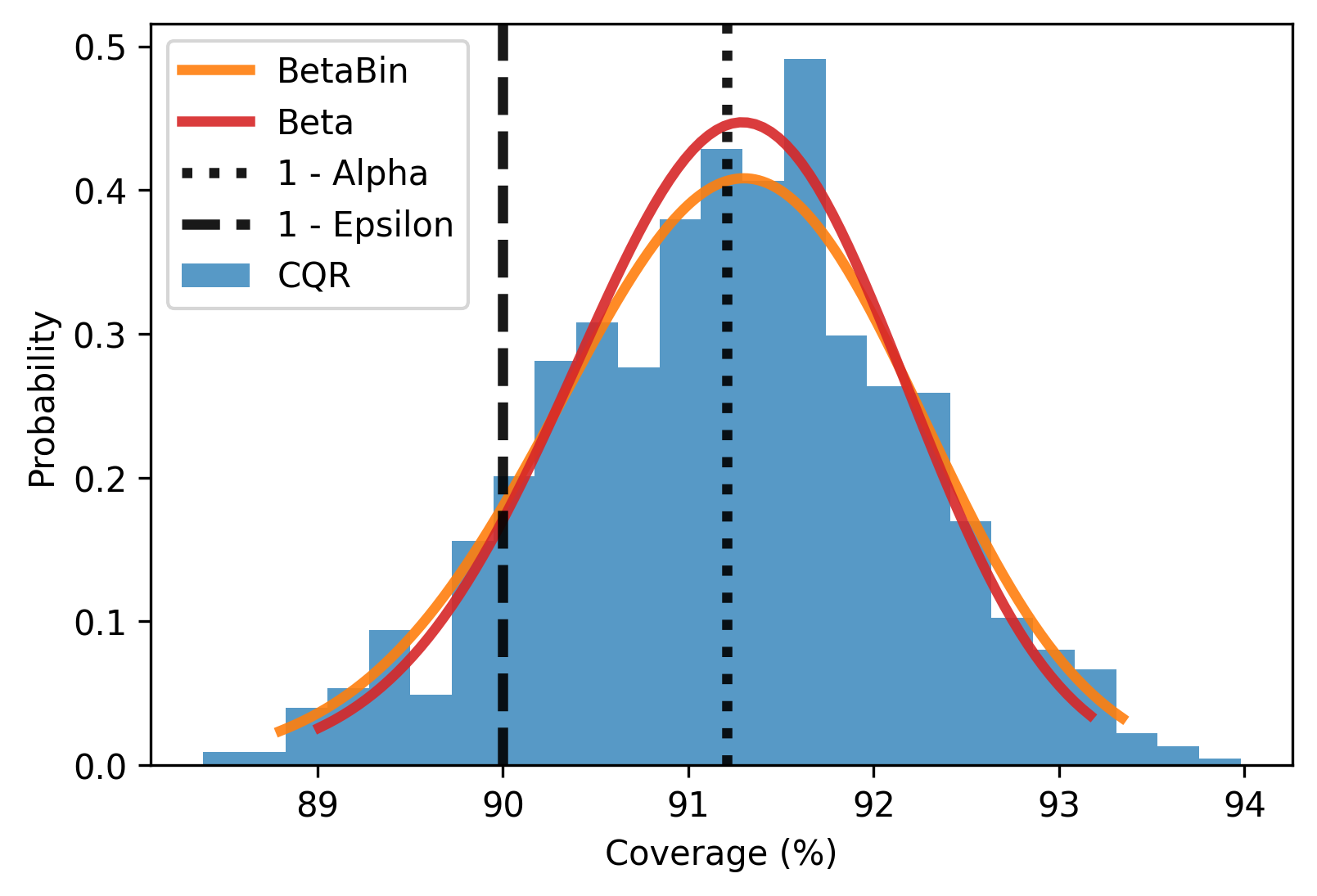}
        \caption{\texttt{synthetic} dataset.}
        \label{fig:toy_cov}
    \end{subfigure}
    \begin{subfigure}[b]{0.45\linewidth}
        \centering
        \includegraphics[width=\linewidth]{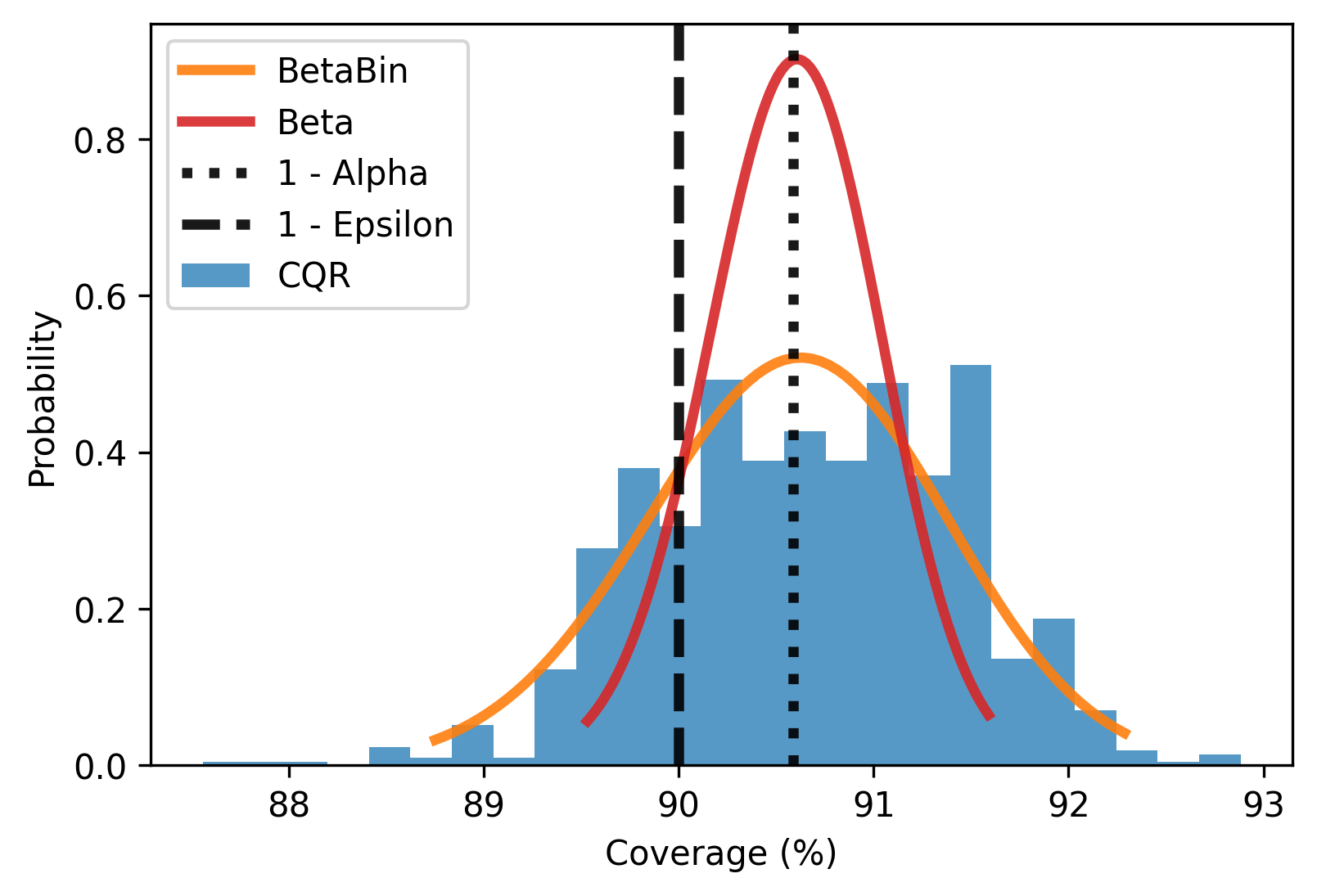}
        \caption{\texttt{bike} dataset.}
        \label{fig:bike_cov}
    \end{subfigure}\\

    \begin{subfigure}[b]{0.45\linewidth}
        \centering
        \includegraphics[width=\linewidth]{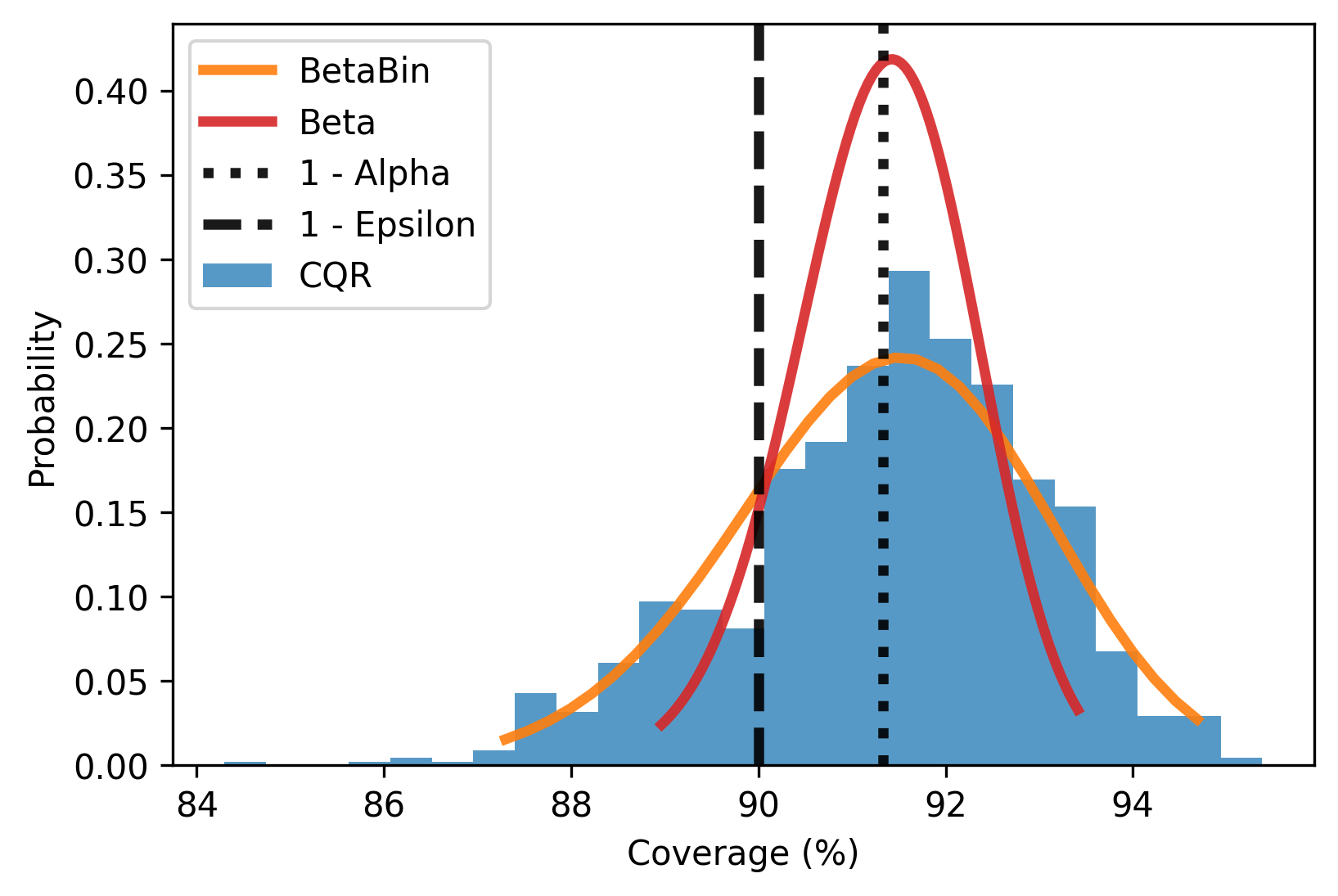}
        \caption{\texttt{star} dataset.}
        \label{fig:star_cov}
    \end{subfigure}
    \begin{subfigure}[b]{0.45\linewidth}
        \centering
        \includegraphics[width=\linewidth]{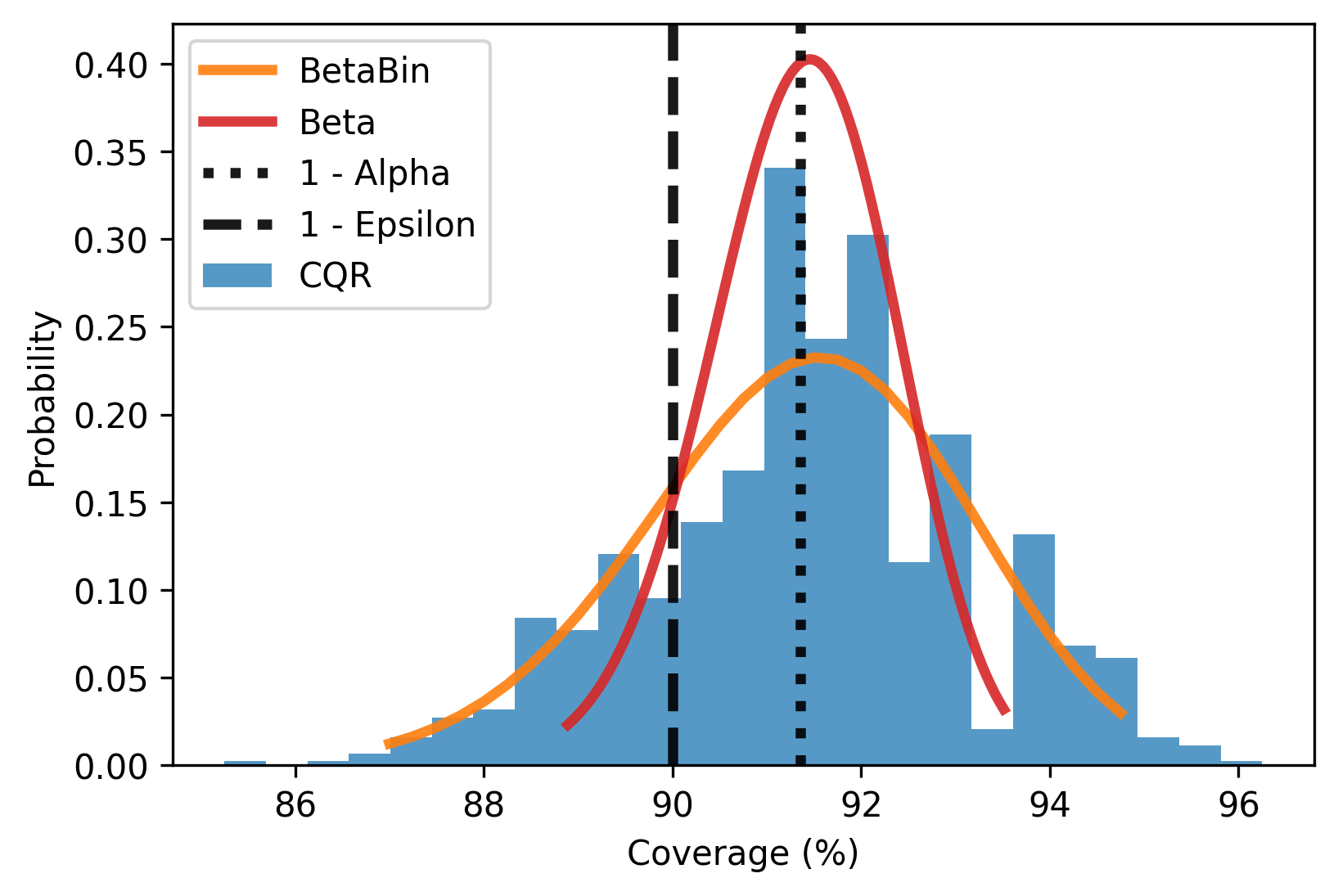}
        \caption{\texttt{community} dataset.}
        \label{fig:com_cov}
    \end{subfigure}\\

    \begin{subfigure}[b]{0.45\linewidth}
        \centering
        \includegraphics[width=\linewidth]{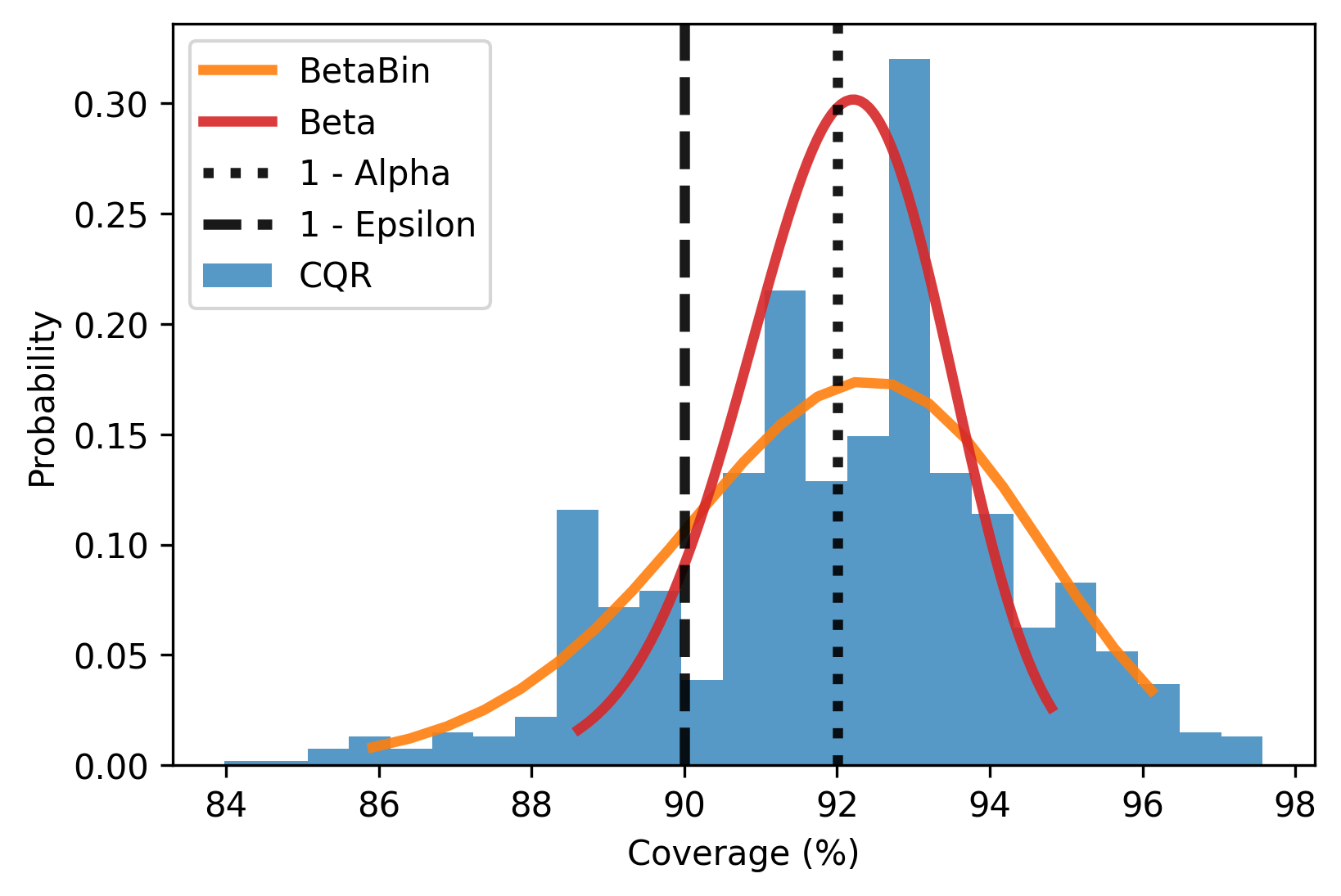}
        \caption{\texttt{concrete} dataset.}
        \label{fig:con_cov}
    \end{subfigure}

    \caption[Empirical coverage of CQR on real datasets.]{Histograms of empirical coverage of CQR prediction sets together with the $\mathrm{Beta}(\lceil (1-\alpha)(n+1)\rceil,\lfloor \alpha(n+1)\rfloor)$ and $\mathrm{BetaBin}(n_\mathrm{test},\lceil (1-\alpha)(n+1)\rceil,\lfloor \alpha(n+1)\rfloor)$ distribution. Note CQR uses 1 QRF base predictor and $R=1000$ random splits of the remaining data into calibration and test set. See \cref{tab:real_cqr} for more details.}
    \label{fig:real_data}
\end{figure}

\cref{tab:real_cqr} reports relevant statistics to \cref{fig:real_data} on average over 10 base predictors. We observe that the average empirical coverage on each dataset is very close to the target coverage. Furthermore, $\overline{\delta}$ is reasonably close to 10\% in most cases, showing CQR produces valid tolerance regions. The prediction sets for the \texttt{bike} and \texttt{community} datasets are too conservative, possibly because the iid assumption is not fully satisfied. 

Finally, we obtain reasonably small interval lengths for most datasets when comparing to the baseline interval length set in \cite{romano2019conformalized}. However, the average interval length for the \texttt{bike} and \texttt{concrete} datasets have almost doubled as compared to \cite{romano2019conformalized}. This might be due to the difference in base predictor and hyper-parameter tuning, but in the case of \texttt{concrete} also partly due to a higher target coverage, as the difference with the fixed $\alpha=0.1$ used in \cite{romano2019conformalized} is almost 2\%. 

\begin{table}[h]
\setlength\tabcolsep{4pt}
\centering
\begin{tabular}{lccccccccc}
\toprule
Dataset & $p$ & $n_\mathrm{train}$ & $n$ & $n_\mathrm{val}$ & $1-\alpha$ (\%) & $\overline{C}$ (\%) & $\delta$ (\%) & $\overline{\delta}$ (\%) & Avg. Length \\
\midrule
\rowcolor{lavender}
\texttt{synthetic} & 1 & 1000 & 1000 & 5000 & 91.21 & 91.22 & 10.00 & 9.78 & 2.18   \\
\texttt{bike} & 18 & 4354 & 4354 & 2178 & 90.59 & 90.65 & 10.00 & 7.75 & 1.13   \\
\rowcolor{lavender}
\texttt{star} & 39 & 864 & 864 & 433 & 91.33 & 91.35 & 10.00 & 10.16 & 0.18   \\
\texttt{community} & 100 & 797 & 797 & 400 & 91.35 & 91.50 & 10.00 & 7.79 & 1.61 \\
\rowcolor{lavender}
\texttt{concrete} & 8 & 412 & 412 & 206 & 92.01 & 92.04 & 10.00 & 10.80 & 1.05   \\
\bottomrule
\end{tabular}
\caption[CQR on real datasets.]{Conformalizing $10$ QRF base predictors each with $R=1000$ trials on a synthetic dataset and 4 real datasets to obtain $(0.1,0.1)$-tolerance regiona. See \cshref{rem:C_bar} and \cshref{rem:delta_bar} for the definition of $\overline{C}$ and $\overline{\delta}$. See \cref{fig:real_data} for a histogram of the empirical coverage of $1$ base predictor.}
\label{tab:real_cqr}
\end{table}

\section{Conclusion}\label{sec:conclusion}

In this thesis, we study marginal coverage and tolerance regions in the context of split conformal prediction, classical tolerance predictors and distribution-free risk control. Although marginal coverage is a widely studied finite-sample guarantee, it is often not the objective of interest in practice. Meanwhile, tolerance regions have received an increasing amount of interest in recent years. We reformulate tolerance regions in the context of split conformal prediction, that have been studied by \cite{vovk2012conditional}, in the language of nested prediction sets of \cite{gupta2022nested}. Furthermore, we extend on the duality between the two finite-sample guarantees, showing that a prediction set that satisfies marginal coverage is an $(\epsilon,\delta)$-tolerance region for certain values of $\epsilon,\delta$, and that a split conformal tolerance region automatically satisfies marginal coverage for certain $\alpha$. How tolerance regions relate to other finite-sample guarantees, such as (class) conditional coverage, is an open question.    

Furthermore, \cite{vovk2012conditional} connected split conformal prediction to the rich literature on classical tolerance predictors, in particular by interpreting split conformal prediction as a `conditional version' of the procedures in \cite{wilks1941determination, scheffe1945non}. That is, if the sample $\{(X_i,Y_i)\}^{n+1}_{i=1}$ is iid, then the non-conformity scores in split conformal prediction can be interpreted as an iid sample from a univariate population so that we can formulate a split conformal prediction set based on univariate order statistics. Drawing from this connection, we prove that the distribution of coverage of a split conformal prediction set, conditional on the calibration set, stochastically dominates the Beta distribution. Whether a similar analytical distribution based on order statistics can be formulated in distribution-free risk control is an interesting direction for future research.

In the context of distribution-free risk control, we demonstrate that split conformal prediction sets satisfying marginal coverage are contained as a special case in CRC \cite{angelopoulos2022conformal} for risks equal to the expectation of the 0-1 loss. Furthermore, for these risks the prediction sets obtained through UCB calibration \cite{bates2021distribution} are identical to the tolerance regions obtained through split conformal prediction. Finally, LTT \cite{angelopoulos2021learn} yields identical prediction sets to UCB calibration and thus split conformal prediction for risks equal to the expectation of the 0-1 loss, i.e. when the objective of interest is tolerance regions. Therefore, in applications where the natural notion of error is miscoverage, we see no reason to prefer any of the above algorithms over split conformal prediction. An unexplored avenue in distribution-free risk control is the relation of conformal risk control to upper confidence bound calibration and whether a unified procedure is able to generalize split conformal prediction with respect to both marginal coverage and tolerance regions. In other words, whether and in what ways \cshref{prop:split_mc_tol} generalizes to monotone risk functions.

Finally, we obtain empirical results in line with theory through a case study of CQR on synthetic and real datasets, identical to the datasets used in \cite{romano2019conformalized}. We demonstrate that we are able to strictly control the coverage of a prediction set so that we obtain an $(\epsilon,\delta)$-tolerance region, while the prediction set sizes roughly coincide with the baseline set in \cite{romano2019conformalized}. Although a more complex base predictor and tuning process might improve efficiency, our experiments illustrate prediction set validity independent of the accuracy of the base predictor. Combining split conformal prediction with efficient base predictors in various contexts is a research direction of great practical importance. 


\addcontentsline{toc}{section}{References}
\bibliographystyle{unsrtnat}
\bibliography{references}

\begin{thebibliography}{87}
\providecommand{\natexlab}[1]{#1}
\providecommand{\url}[1]{\texttt{#1}}
\expandafter\ifx\csname urlstyle\endcsname\relax
  \providecommand{\doi}[1]{doi: #1}\else
  \providecommand{\doi}{doi: \begingroup \urlstyle{rm}\Url}\fi

\bibitem[Vovk et~al.(2005)Vovk, Gammerman, and Shafer]{vovk2005algorithmic}
V.~Vovk, A.~Gammerman, and G.~Shafer.
\newblock \emph{Algorithmic Learning in a Random World}.
\newblock Springer, New York, 2005.
\newblock \doi{10.1007/b106715}.

\bibitem[Vovk et~al.(2009)Vovk, Nouretdinov, and Gammerman]{vovk2009line}
V.~Vovk, I.~Nouretdinov, and A.~Gammerman.
\newblock On-line predictive linear regression.
\newblock \emph{The Annals of Statistics}, 37\penalty0 (3):\penalty0
  1566--1590, 2009.
\newblock \doi{10.1214/08-AOS622}.

\bibitem[Lei et~al.(2018)Lei, G’Sell, Rinaldo, Tibshirani, and
  Wasserman]{lei2018distribution}
J.~Lei, M.~G’Sell, A.~Rinaldo, R.~J. Tibshirani, and L.~Wasserman.
\newblock Distribution-free predictive inference for regression.
\newblock \emph{Journal of the American Statistical Association}, 113\penalty0
  (523):\penalty0 1094--1111, 2018.
\newblock \doi{10.1080/01621459.2017.1307116}.

\bibitem[Papadopoulos et~al.(2002)Papadopoulos, Proedrou, Vovk, and
  Gammerman]{papadopoulos2002inductive}
H.~Papadopoulos, K.~Proedrou, V.~Vovk, and A.~Gammerman.
\newblock Inductive confidence machines for regression.
\newblock In \emph{Machine Learning: European Conference on Machine Learning},
  pages 345--356, 2002.
\newblock \doi{10.1007/3-540-36755-1_29}.

\bibitem[Angelopoulos and Bates(2022)]{angelopoulos2021gentle}
A.~N Angelopoulos and S.~Bates.
\newblock A gentle introduction to conformal prediction and distribution-free
  uncertainty quantification.
\newblock \emph{arXiv preprint arXiv:2107.07511}, 2022.
\newblock \doi{10.48550/arXiv.2107.07511}.

\bibitem[Shafer and Vovk(2008)]{shafer2008tutorial}
G.~Shafer and V.~Vovk.
\newblock A tutorial on conformal prediction.
\newblock \emph{Journal of Machine Learning Research}, 9\penalty0
  (12):\penalty0 371--421, 2008.

\bibitem[Alvarsson et~al.(2021)Alvarsson, McShane, Norinder, and
  Spjuth]{alvarsson2021predicting}
J.~Alvarsson, S.~A. McShane, U.~Norinder, and O.~Spjuth.
\newblock Predicting with confidence: using conformal prediction in drug
  discovery.
\newblock \emph{Journal of Pharmaceutical Sciences}, 110\penalty0 (1):\penalty0
  42--49, 2021.
\newblock \doi{10.1016/j.xphs.2020.09.055}.

\bibitem[Angelopoulos et~al.(2020)Angelopoulos, Bates, Malik, and
  Jordan]{angelopoulos2020uncertainty}
A.~N. Angelopoulos, S.~Bates, J.~Malik, and M.~I. Jordan.
\newblock Uncertainty sets for image classifiers using conformal prediction.
\newblock \emph{arXiv preprint arXiv:2009.14193}, 2020.
\newblock \doi{10.48550/arXiv.2009.14193}.

\bibitem[Fisch et~al.(2021{\natexlab{a}})Fisch, Schuster, Jaakkola, and
  Barzilay]{fisch2020efficient}
A.~Fisch, T.~Schuster, T.~Jaakkola, and R.~Barzilay.
\newblock Efficient conformal prediction via cascaded inference with expanded
  admission.
\newblock \emph{arXiv preprint arXiv:2007.03114}, 2021{\natexlab{a}}.
\newblock \doi{10.48550/arXiv.2007.03114}.

\bibitem[Fisch et~al.(2021{\natexlab{b}})Fisch, Schuster, Jaakkola, and
  Barzilay]{fisch2021few}
A.~Fisch, T.~Schuster, T.~Jaakkola, and R.~Barzilay.
\newblock Few-shot conformal prediction with auxiliary tasks.
\newblock In \emph{Proceedings of Machine Learning Research}, volume 139, pages
  3329--3339, 2021{\natexlab{b}}.

\bibitem[Schuster et~al.(2021)Schuster, Fisch, Jaakkola, and
  Barzilay]{schuster2021consistent}
T.~Schuster, A.~Fisch, T.~Jaakkola, and R.~Barzilay.
\newblock Consistent accelerated inference via confident adaptive transformers.
\newblock \emph{arXiv preprint arXiv:2104.08803}, 2021.
\newblock \doi{10.48550/arXiv.2104.08803}.

\bibitem[Cherian and Bronner(2020)]{cherian2020washington}
J.~Cherian and L.~Bronner.
\newblock How the washington post estimates outstanding votes for the 2020
  presidential election.
\newblock \emph{Washington Post}, 2020.
\newblock URL
  \url{https://s3.us-east-1.amazonaws.com/elex-models-prod/2020-general/write-up/election_model_writeup.pdf}.

\bibitem[Gupta et~al.(2022)Gupta, Kuchibhotla, and Ramdas]{gupta2022nested}
C.~Gupta, A.~K. Kuchibhotla, and A.~Ramdas.
\newblock Nested conformal prediction and quantile out-of-bag ensemble methods.
\newblock \emph{Pattern Recognition}, 127\penalty0 (108496), 2022.
\newblock \doi{10.1016/j.patcog.2021.108496}.

\bibitem[Vovk(2013)]{vovk2012conditional}
V.~Vovk.
\newblock Conditional validity of inductive conformal predictors.
\newblock \emph{Machine Learning}, 92:\penalty0 349--376, 2013.
\newblock \doi{10.1007/s10994-013-5355-6}.

\bibitem[Valiant(1984)]{valiant1984theory}
L.~G. Valiant.
\newblock A theory of the learnable.
\newblock \emph{Communications of the ACM}, 27\penalty0 (11):\penalty0
  1134--1142, 1984.
\newblock \doi{10.1145/1968.1972}.

\bibitem[Park et~al.(2020)Park, Bastani, Matni, and Lee]{park2019pac}
S.~Park, O.~Bastani, N.~Matni, and I.~Lee.
\newblock {PAC} confidence sets for deep neural networks via calibrated
  prediction.
\newblock \emph{arXiv preprint arXiv:2001.00106}, 2020.
\newblock \doi{10.48550/arXiv.2001.00106}.

\bibitem[Park et~al.(2021)Park, Li, Lee, and Bastani]{park2020pac}
S.~Park, S.~Li, I.~Lee, and O.~Bastani.
\newblock {PAC} confidence predictions for deep neural network classifiers.
\newblock \emph{arXiv preprint arXiv:2011.00716}, 2021.
\newblock \doi{10.48550/arXiv.2011.00716}.

\bibitem[Vapnik(1999)]{vapnik1999overview}
V.~N. Vapnik.
\newblock An overview of statistical learning theory.
\newblock \emph{IEEE transactions on neural networks}, 10\penalty0
  (5):\penalty0 988--999, 1999.
\newblock \doi{10.1109/72.788640}.

\bibitem[Wilks(1941)]{wilks1941determination}
S.~S. Wilks.
\newblock Determination of sample sizes for setting tolerance limits.
\newblock \emph{The Annals of Mathematical Statistics}, 12\penalty0
  (1):\penalty0 91--96, 1941.
\newblock \doi{10.1214/aoms/1177731788}.

\bibitem[Wald(1943)]{wald1943extension}
A.~Wald.
\newblock An extension of {W}ilks' method for setting tolerance limits.
\newblock \emph{The Annals of Mathematical Statistics}, 14\penalty0
  (1):\penalty0 45--55, 1943.
\newblock \doi{10.1214/aoms/1177731491}.

\bibitem[Scheffe and Tukey(1945)]{scheffe1945non}
H.~Scheffe and J.~W. Tukey.
\newblock Non-parametric estimation. {I}. {V}alidation of order statistics.
\newblock \emph{The Annals of Mathematical Statistics}, 16\penalty0
  (2):\penalty0 187--192, 1945.
\newblock \doi{10.1214/aoms/1177731119}.

\bibitem[Tukey(1947)]{tukey1947non}
J.~W. Tukey.
\newblock Non-parametric estimation {II}. {S}tatistically equivalent blocks and
  tolerance regions -- {T}he continuous case.
\newblock \emph{The Annals of Mathematical Statistics}, 18\penalty0
  (4):\penalty0 529--539, 1947.
\newblock \doi{10.1214/aoms/1177730343}.

\bibitem[Tukey(1948)]{tukey1948nonparametric}
J.~W. Tukey.
\newblock Nonparametric estimation, {III}. {S}tatistically equivalent blocks
  and multivariate tolerance regions -- {T}he discontinuous case.
\newblock \emph{The Annals of Mathematical Statistics}, 19\penalty0
  (1):\penalty0 30--39, 1948.
\newblock \doi{10.1214/aoms/1177730287}.

\bibitem[Fraser and Wormleighton(1951)]{fraser1951nonparametric}
D.~A.~S. Fraser and R.~Wormleighton.
\newblock Nonparametric estimation {IV}.
\newblock \emph{The Annals of Mathematical Statistics}, 22\penalty0
  (2):\penalty0 294--298, 1951.
\newblock \doi{10.1214/aoms/1177729650}.

\bibitem[Fraser(1951)]{fraser1951sequentially}
D.~A.~S. Fraser.
\newblock Sequentially determined statistically equivalent blocks.
\newblock \emph{The Annals of Mathematical Statistics}, 22\penalty0
  (3):\penalty0 372--381, 1951.
\newblock \doi{10.1214/aoms/1177729583}.

\bibitem[Fraser(1953)]{fraser1953nonparametric}
D.~A.~S. Fraser.
\newblock Nonparametric tolerance regions.
\newblock \emph{The Annals of Mathematical Statistics}, 24\penalty0
  (1):\penalty0 44--55, 1953.
\newblock \doi{10.1214/aoms/1177729081}.

\bibitem[Kemperman(1956)]{kemperman1956generalized}
J.~H.~B. Kemperman.
\newblock Generalized tolerance limits.
\newblock \emph{The Annals of Mathematical Statistics}, 27\penalty0
  (1):\penalty0 180--186, 1956.
\newblock \doi{10.1214/aoms/1177728356}.

\bibitem[Guttman(1970)]{guttman1970statistical}
I.~Guttman.
\newblock \emph{Statistical Tolerance Regions: Classical and Bayesian}.
\newblock Griffin, London, 1970.

\bibitem[Krishnamoorthy and Mathew(2009)]{krishnamoorthy2009statistical}
K.~Krishnamoorthy and T.~Mathew.
\newblock \emph{Statistical Tolerance Regions: Theory, Applications, and
  Computation}.
\newblock Wiley, New York, 2009.

\bibitem[Angelopoulos et~al.(2022{\natexlab{a}})Angelopoulos, Bates, Fisch,
  Lei, and Schuster]{angelopoulos2022conformal}
A.~N. Angelopoulos, S.~Bates, A.~Fisch, L.~Lei, and T.~Schuster.
\newblock Conformal risk control.
\newblock \emph{arXiv preprint arXiv:2208.02814}, 2022{\natexlab{a}}.
\newblock \doi{10.48550/arXiv.2208.02814}.

\bibitem[Bates et~al.(2021)Bates, Angelopoulos, Lei, Malik, and
  Jordan]{bates2021distribution}
S.~Bates, A.~N. Angelopoulos, L.~Lei, J.~Malik, and M.~I. Jordan.
\newblock Distribution-free, risk-controlling prediction sets.
\newblock \emph{Journal of the ACM}, 68\penalty0 (6):\penalty0 1--34, 2021.
\newblock \doi{10.1145/3478535}.

\bibitem[Angelopoulos et~al.(2022{\natexlab{b}})Angelopoulos, Bates, Cand{è}s,
  Jordan, and Lei]{angelopoulos2021learn}
A.~N. Angelopoulos, S.~Bates, E.~J. Cand{è}s, M.~I. Jordan, and L.~Lei.
\newblock Learn then test: Calibrating predictive algorithms to achieve risk
  control.
\newblock \emph{arXiv preprint arXiv:2110.01052}, 2022{\natexlab{b}}.
\newblock \doi{10.48550/arXiv.2110.01052}.

\bibitem[Romano et~al.(2019)Romano, Patterson, and
  Cand{è}s]{romano2019conformalized}
Y.~Romano, E.~Patterson, and E.~J. Cand{è}s.
\newblock Conformalized quantile regression.
\newblock In \emph{Advances in Neural Information Processing Systems},
  volume~32, pages 3543--3553, 2019.

\bibitem[Meinshausen(2006)]{meinshausen2006quantile}
N.~Meinshausen.
\newblock Quantile regression forests.
\newblock \emph{The Journal of Machine Learning Research}, 7\penalty0
  (35):\penalty0 983--999, 2006.

\bibitem[Tibshirani et~al.(2019)Tibshirani, Barber, Cand{è}s, and
  Ramdas]{tibshirani2019conformal}
R.~J. Tibshirani, R.~F. Barber, E.~J. Cand{è}s, and A.~Ramdas.
\newblock Conformal prediction under covariate shift.
\newblock In \emph{Advances in neural information processing systems},
  volume~32, pages 2530--2540, 2019.

\bibitem[Park et~al.(2022)Park, Dobriban, Lee, and Bastani]{park2021pac}
S.~Park, E.~Dobriban, I.~Lee, and O.~Bastani.
\newblock {PAC} prediction sets under covariate shift.
\newblock \emph{arXiv preprint arXiv:2106.09848}, 2022.
\newblock \doi{10.48550/arXiv.2106.09848}.

\bibitem[Qiu et~al.(2022)Qiu, Dobriban, and Tchetgen]{qiu2022distribution}
H.~Qiu, E.~Dobriban, and E.~T. Tchetgen.
\newblock Distribution-free prediction sets adaptive to unknown covariate
  shift.
\newblock \emph{arXiv preprint arXiv:2203.06126}, 2022.
\newblock \doi{10.48550/arXiv.2203.06126}.

\bibitem[Podkopaev and Ramdas(2021)]{podkopaev2021distribution}
A.~Podkopaev and A.~Ramdas.
\newblock Distribution-free uncertainty quantification for classification under
  label shift.
\newblock In \emph{Proceedings of Machine Learning Research}, volume 161, pages
  844--853, 2021.

\bibitem[Cauchois et~al.(2020)Cauchois, Gupta, Ali, and
  Duchi]{cauchois2020robust}
M.~Cauchois, S.~Gupta, A.~Ali, and J.~C. Duchi.
\newblock Robust validation: Confident predictions even when distributions
  shift.
\newblock \emph{arXiv preprint arXiv:2008.04267}, 2020.
\newblock \doi{10.48550/arXiv.2008.04267}.

\bibitem[Gibbs and Cand{è}s(2021)]{gibbs2021adaptive}
I.~Gibbs and E.~J. Cand{è}s.
\newblock Adaptive conformal inference under distribution shift.
\newblock In \emph{Advances in Neural Information Processing Systems},
  volume~34, pages 1660--1672, 2021.

\bibitem[Barber et~al.(2022)Barber, Cand{è}s, Ramdas, and
  Tibshirani]{barber2022conformal}
R.~F. Barber, E.~J. Cand{è}s, A.~Ramdas, and R.~J. Tibshirani.
\newblock Conformal prediction beyond exchangeability.
\newblock \emph{arXiv preprint arXiv:2202.13415}, 2022.
\newblock \doi{10.48550/arXiv.2202.13415}.

\bibitem[Lei and Cand{è}s()]{lei2021conformal}
L.~Lei and E.~J. Cand{è}s.
\newblock Conformal inference of counterfactuals and individual treatment
  effects.
\newblock \emph{Journal of the Royal Statistical Society: Series B (Statistical
  Methodology)}, 83\penalty0 (5):\penalty0 911--938.
\newblock \doi{10.1111/rssb.12445}.

\bibitem[Yin et~al.(2022)Yin, Shi, Wang, and Blei]{yin2022conformal}
M.~Yin, C.~Shi, Y.~Wang, and D.~M. Blei.
\newblock Conformal sensitivity analysis for individual treatment effects.
\newblock \emph{Journal of the American Statistical Association}, pages 1--30,
  2022.
\newblock \doi{10.1080/01621459.2022.2102503}.

\bibitem[Chernozhukov et~al.(2021)Chernozhukov, W{\"u}thrich, and
  Zhu]{chernozhukov2021exact}
V.~Chernozhukov, K.~W{\"u}thrich, and Y.~Zhu.
\newblock An exact and robust conformal inference method for counterfactual and
  synthetic controls.
\newblock \emph{Journal of the American Statistical Association}, 116\penalty0
  (536):\penalty0 1849--1864, 2021.
\newblock \doi{10.1080/01621459.2021.1920957}.

\bibitem[Cand{\`e}s et~al.(2022)Cand{\`e}s, Lei, and
  Ren]{candes2021conformalized}
E.~J. Cand{\`e}s, L.~Lei, and Z.~Ren.
\newblock Conformalized survival analysis.
\newblock \emph{arXiv preprint arXiv:2103.09763}, 2022.
\newblock \doi{10.48550/arXiv.2103.09763}.

\bibitem[Angelopoulos et~al.(2021)Angelopoulos, Bates, Zrnic, and
  Jordan]{angelopoulos2021private}
A.~N. Angelopoulos, S.~Bates, T.~Zrnic, and M.~I. Jordan.
\newblock Private prediction sets.
\newblock \emph{arXiv preprint arXiv:2102.06202}, 2021.
\newblock \doi{10.48550/arXiv.2102.06202}.

\bibitem[Fannjiang et~al.(2022)Fannjiang, Bates, Angelopoulos, Listgarten, and
  Jordan]{fannjiang2022conformal}
C.~Fannjiang, S.~Bates, A.~N. Angelopoulos, J.~Listgarten, and M.~I. Jordan.
\newblock Conformal prediction for the design problem.
\newblock \emph{arXiv preprint arXiv:2202.03613}, 2022.
\newblock \doi{10.48550/arXiv.2202.03613}.

\bibitem[Mao et~al.(2020)Mao, Martin, and Reich]{mao2020valid}
H.~Mao, R.~Martin, and B.~Reich.
\newblock Valid model-free spatial prediction.
\newblock \emph{arXiv preprint arXiv:2006.15640}, 2020.
\newblock \doi{10.48550/arXiv.2006.15640}.

\bibitem[Chernozhukov et~al.(2018)Chernozhukov, W{\"u}thrich, and
  Yinchu]{chernozhukov2018exact}
V.~Chernozhukov, K.~W{\"u}thrich, and Z.~Yinchu.
\newblock Exact and robust conformal inference methods for predictive machine
  learning with dependent data.
\newblock In \emph{Proceedings of Machine Learning Research}, volume~75, pages
  732--749, 2018.

\bibitem[Dunn et~al.(2022)Dunn, Wasserman, and Ramdas]{dunn2018distribution}
R.~Dunn, L.~Wasserman, and A.~Ramdas.
\newblock Distribution-free prediction sets for two-layer hierarchical models.
\newblock \emph{arXiv preprint arXiv:1809.07441}, 2022.
\newblock \doi{10.48550/arXiv.1809.07441}.

\bibitem[Oliveira et~al.(2022)Oliveira, Orenstein, Ramos, and
  Romano]{oliveira2022split}
R.~I. Oliveira, P.~Orenstein, T.~Ramos, and J.~V. Romano.
\newblock Split conformal prediction for dependent data.
\newblock \emph{arXiv preprint arXiv:2203.15885}, 2022.
\newblock \doi{10.48550/arXiv.2203.15885}.

\bibitem[Taufiq et~al.(2022)Taufiq, Ton, Cornish, Teh, and
  Doucet]{taufiq2022conformal}
M.~F. Taufiq, J.~Ton, R.~Cornish, Y.~W. Teh, and A.~Doucet.
\newblock Conformal off-policy prediction in contextual bandits.
\newblock \emph{arXiv preprint arXiv:2206.04405}, 2022.
\newblock \doi{10.48550/arXiv.2206.04405}.

\bibitem[Fraser(1957)]{fraser1956nonparametric}
D.~A.~S. Fraser.
\newblock \emph{Nonparametric Methods in Statistics.}
\newblock Wiley, New York, 1957.

\bibitem[Lei and Wasserman(2014)]{lei2014distribution}
J.~Lei and L.~Wasserman.
\newblock Distribution-free prediction bands for non-parametric regression.
\newblock \emph{Journal of the Royal Statistical Society: Series B (Statistical
  Methodology)}, 76\penalty0 (1):\penalty0 71--96, 2014.
\newblock \doi{10.1111/rssb.12021}.

\bibitem[Izbicki et~al.(2019)Izbicki, Shimizu, and Stern]{izbicki2019flexible}
R.~Izbicki, G.~T. Shimizu, and R.~B. Stern.
\newblock Flexible distribution-free conditional predictive bands using density
  estimators.
\newblock \emph{arXiv preprint arXiv:1910.05575}, 2019.
\newblock \doi{10.48550/arXiv.1910.05575}.

\bibitem[Guan(2020)]{guan2020conformal}
L.~Guan.
\newblock Conformal prediction with localization.
\newblock \emph{arXiv preprint arXiv:1908.08558}, 2020.
\newblock \doi{10.48550/arXiv.1908.08558}.

\bibitem[Romano et~al.(2020{\natexlab{a}})Romano, Sesia, and
  Cand{è}s]{romano2020classification}
Y.~Romano, M.~Sesia, and E.~J. Cand{è}s.
\newblock Classification with valid and adaptive coverage.
\newblock In \emph{Advances in Neural Information Processing Systems},
  volume~33, pages 3581--3591, 2020{\natexlab{a}}.

\bibitem[Cauchois et~al.(2021)Cauchois, Gupta, and Duchi]{cauchois2021knowing}
M.~Cauchois, S.~Gupta, and J.~C. Duchi.
\newblock Knowing what you know: {V}alid and validated confidence sets in
  multiclass and multilabel prediction.
\newblock \emph{Journal of Machine Learning Research}, 22\penalty0
  (81):\penalty0 1--42, 2021.

\bibitem[Barber et~al.(2021{\natexlab{a}})Barber, Cand{è}s, Ramdas, and
  Tibshirani]{foygel2021limits}
R.~F. Barber, E.~J. Cand{è}s, A.~Ramdas, and R.~J. Tibshirani.
\newblock The limits of distribution-free conditional predictive inference.
\newblock \emph{Information and Inference: A Journal of the IMA}, 10\penalty0
  (2):\penalty0 455--482, 2021{\natexlab{a}}.
\newblock \doi{10.1093/imaiai/iaaa017}.

\bibitem[Lei(2014)]{lei2014classification}
J.~Lei.
\newblock Classification with confidence.
\newblock \emph{Biometrika}, 101\penalty0 (4):\penalty0 755--769, 2014.
\newblock \doi{10.1093/biomet/asu038}.

\bibitem[Hechtlinger et~al.(2018)Hechtlinger, P{\'o}czos, and
  Wasserman]{hechtlinger2018cautious}
Y.~Hechtlinger, B.~P{\'o}czos, and L.~Wasserman.
\newblock Cautious deep learning.
\newblock \emph{arXiv preprint arXiv:1805.09460}, 2018.
\newblock \doi{10.48550/arXiv.1805.09460}.

\bibitem[Sadinle et~al.(2019)Sadinle, Lei, and Wasserman]{sadinle2019least}
M.~Sadinle, J.~Lei, and L.~Wasserman.
\newblock Least ambiguous set-valued classifiers with bounded error levels.
\newblock \emph{Journal of the American Statistical Association}, 114\penalty0
  (525):\penalty0 223--234, 2019.
\newblock \doi{10.1080/01621459.2017.1395341}.

\bibitem[Romano et~al.(2020{\natexlab{b}})Romano, Barber, Sabatti, and
  Cand{è}s]{romano2019malice}
Y.~Romano, R.~F. Barber, C.~Sabatti, and E.~J. Cand{è}s.
\newblock With malice towards none: Assessing uncertainty via equalized
  coverage.
\newblock \emph{Harvard Data Science Review}, 2\penalty0 (2),
  2020{\natexlab{b}}.
\newblock \doi{10.1162/99608f92.03f00592}.

\bibitem[Guan and Tibshirani(2022)]{guan2022prediction}
L.~Guan and R.~J. Tibshirani.
\newblock Prediction and outlier detection in classification problems.
\newblock \emph{Journal of the Royal Statistical Society: Series B (Statistical
  Methodology)}, 84\penalty0 (2):\penalty0 524--546, 2022.
\newblock \doi{10.1111/rssb.12443}.

\bibitem[Fraser and Guttman(1956)]{fraser1956tolerance}
D.~A.~S. Fraser and I.~Guttman.
\newblock Tolerance regions.
\newblock \emph{The Annals of Mathematical Statistics}, 27\penalty0
  (1):\penalty0 162--179, 1956.
\newblock \doi{10.1214/aoms/1177728355}.

\bibitem[Vovk(2015)]{vovk2015cross}
V.~Vovk.
\newblock Cross-conformal predictors.
\newblock \emph{Annals of Mathematics and Artificial Intelligence}, 74\penalty0
  (1):\penalty0 9--28, 2015.
\newblock \doi{10.1007/s10472-013-9368-4}.

\bibitem[Barber et~al.(2021{\natexlab{b}})Barber, Cand{è}s, Ramdas, and
  Tibshirani]{barber2021predictive}
R.~F. Barber, E.~J. Cand{è}s, A.~Ramdas, and R.~J. Tibshirani.
\newblock Predictive inference with the jackknife+.
\newblock \emph{The Annals of Statistics}, 49\penalty0 (1):\penalty0 486--507,
  2021{\natexlab{b}}.
\newblock \doi{10.1214/20-AOS1965}.

\bibitem[Johansson et~al.(2014)Johansson, Bostr{\"o}m, L{\"o}fstr{\"o}m, and
  Linusson]{johansson2014regression}
U.~Johansson, H.~Bostr{\"o}m, T.~L{\"o}fstr{\"o}m, and H.~Linusson.
\newblock Regression conformal prediction with random forests.
\newblock \emph{Machine learning}, 97\penalty0 (1):\penalty0 155--176, 2014.
\newblock \doi{10.1007/s10994-014-5453-0}.

\bibitem[Bostr{\"o}m et~al.(2017)Bostr{\"o}m, Linusson, L{\"o}fstr{\"o}m, and
  Johansson]{bostrom2017accelerating}
H.~Bostr{\"o}m, H.~Linusson, T.~L{\"o}fstr{\"o}m, and U.~Johansson.
\newblock Accelerating difficulty estimation for conformal regression forests.
\newblock \emph{Annals of Mathematics and Artificial Intelligence}, 81\penalty0
  (1):\penalty0 125--144, 2017.
\newblock \doi{10.1007/s10472-017-9539-9}.

\bibitem[Linusson et~al.(2020)Linusson, Johansson, and
  Bostr{\"o}m]{linusson2020efficient}
H.~Linusson, U.~Johansson, and H.~Bostr{\"o}m.
\newblock Efficient conformal predictor ensembles.
\newblock \emph{Neurocomputing}, 397:\penalty0 266--278, 2020.
\newblock \doi{10.1016/j.neucom.2019.07.113}.

\bibitem[Kim et~al.(2020)Kim, Xu, and Barber]{kim2020predictive}
B.~Kim, C.~Xu, and R.~F. Barber.
\newblock Predictive inference is free with the jackknife+-after-bootstrap.
\newblock In \emph{Advances in Neural Information Processing Systems},
  volume~33, pages 4138--4149, 2020.

\bibitem[Koenker and Bassett~Jr(1978)]{koenker1978regression}
R.~Koenker and G.~Bassett~Jr.
\newblock Regression quantiles.
\newblock \emph{Econometrica}, 46\penalty0 (1):\penalty0 33--50, 1978.
\newblock \doi{10.2307/1913643}.

\bibitem[Bassett~Jr and Koenker(1982)]{bassett1982empirical}
G.~Bassett~Jr and R.~Koenker.
\newblock An empirical quantile function for linear models with iid errors.
\newblock \emph{Journal of the American Statistical Association}, 77\penalty0
  (378):\penalty0 407--415, 1982.
\newblock \doi{10.1080/01621459.1982.10477826}.

\bibitem[Takeuchi et~al.(2006)Takeuchi, Le, Sears, and
  Smola]{takeuchi2006nonparametric}
I.~Takeuchi, Q.~V. Le, T.~D. Sears, and A.~J. Smola.
\newblock Nonparametric quantile estimation.
\newblock \emph{Journal of Machine Learning Research}, 7\penalty0
  (45):\penalty0 1231--1264, 2006.

\bibitem[Steinwart and Christmann(2011)]{steinwart2011estimating}
I.~Steinwart and A.~Christmann.
\newblock Estimating conditional quantiles with the help of the pinball loss.
\newblock \emph{Bernoulli}, 17\penalty0 (1):\penalty0 211--225, 2011.
\newblock \doi{10.3150/10-BEJ267}.

\bibitem[Hoeffding(1963)]{hoeffding1994probability}
W.~Hoeffding.
\newblock Probability inequalities for sums of bounded random variables.
\newblock \emph{Journal of the American Statistical Association}, 58\penalty0
  (301):\penalty0 13--30, 1963.
\newblock \doi{10.1080/01621459.1963.10500830}.

\bibitem[Bentkus(2004)]{bentkus2004hoeffding}
V.~Bentkus.
\newblock On {H}oeffding’s inequalities.
\newblock \emph{The Annals of Probability}, 32\penalty0 (2):\penalty0
  1650--1673, 2004.
\newblock \doi{10.1214/009117904000000360}.

\bibitem[Bauer(1991)]{bauer1991multiple}
P.~Bauer.
\newblock Multiple testing in clinical trials.
\newblock \emph{Statistics in Medicine}, 10\penalty0 (6):\penalty0 871--890,
  1991.
\newblock \doi{10.1002/sim.4780100609}.

\bibitem[Bonferroni(1936)]{bonferroni1936teoria}
C.~E. Bonferroni.
\newblock Teoria statistica delle classi e calcolo delle probabilita.
\newblock \emph{Pubblicazioni del R Istituto Superiore di Scienze Economiche e
  Commericiali di Firenze}, 8:\penalty0 3--62, 1936.

\bibitem[Holm(1979)]{holm1979simple}
S.~Holm.
\newblock A simple sequentially rejective multiple test procedure.
\newblock \emph{Scandinavian Journal of Statistics}, 6\penalty0 (2):\penalty0
  65--70, 1979.

\bibitem[Fanaee-T and Gama(2013)]{bike}
H.~Fanaee-T and J.~Gama.
\newblock Bike sharing dataset data set.
\newblock \url{https://archive.ics.uci.edu/ml/datasets/bike+sharing+dataset},
  2013.
\newblock Citation request: \cite{fanaee2014event}.

\bibitem[Achilles et~al.(2008)Achilles, Bain, Bellott, Boyd-Zaharias, Finn,
  Folger, Johnston, and Word]{star}
C.~M. Achilles, H.~P. Bain, F.~Bellott, J.~Boyd-Zaharias, J.~Finn, J.~Folger,
  J.~Johnston, and E.~Word.
\newblock Tennessee's student teacher achievement ratio ({STAR}) project, 2008.
\newblock
  \url{https://dataverse.harvard.edu/citation?persistentId=doi:10.7910/DVN/I9E0OF}.

\bibitem[Redmond(2009)]{community}
M.~Redmond.
\newblock Communities and crime data set.
\newblock \url{http://archive.ics.uci.edu/ml/datasets/communities+and+crime},
  2009.
\newblock Citation request: \cite{remond2002data}.

\bibitem[Yeh(2007)]{concrete}
I.-C. Yeh.
\newblock Concrete compressive strength data set.
\newblock
  \url{http://archive.ics.uci.edu/ml/datasets/concrete+compressive+strength},
  2007.
\newblock Citation request: \cite{yeh1998modelling}.

\bibitem[Fanaee-T and Gama(2014)]{fanaee2014event}
H.~Fanaee-T and J.~Gama.
\newblock Event labeling combining ensemble detectors and background knowledge.
\newblock \emph{Progress in Artificial Intelligence}, \penalty0 (2):\penalty0
  113–127, 2014.
\newblock \doi{10.1007/s13748-013-0040-3}.

\bibitem[Remond and Baveja(2002)]{remond2002data}
M.~Remond and A.~Baveja.
\newblock A data-driven software tool for enabling cooperative information
  sharing among police departments.
\newblock \emph{European Journal of Operational Research}, 141\penalty0
  (3):\penalty0 660--678, 2002.
\newblock \doi{10.1016/S0377-2217(01)00264-8}.

\bibitem[Yeh(1998)]{yeh1998modelling}
I.-C. Yeh.
\newblock Modeling of strength of high-performance concrete using artificial
  neural networks.
\newblock \emph{Cement and Concrete Research}, 28\penalty0 (12):\penalty0
  1797--1808, 1998.
\newblock \doi{10.1016/S0008-8846(98)00165-3}.

\end{thebibliography}

\clearpage

\appendix

\section{Proofs}\label{app:proofs}

\subsection{Proposition 1}\label{app:prop1}
\begin{proof}
See the proof of \cite[Theorem 1]{romano2019conformalized}. To address some notational differences, $n$ denotes the size of the full training set, $\mathcal{I}_1$ and $\mathcal{I}_2$ denote the set of indices of respectively the proper training and calibration set, $C(X_{n+1})$ denotes $\mathcal{S}_{\widehat{\lambda}_\mathrm{split}}(X_{n+1})$, the $E_i$ denote the conformity scores $r_i$ and $Q_{1-\alpha}(E,\mathcal{I}_2)$ denotes $\widehat{\lambda}_\mathrm{split}=\widehat{Q}(\alpha)$. Finally, note that the result also holds under the weaker assumption of exchangeability and both conditional and unconditional on the proper training set. This concludes the proof.

\end{proof}

\subsection{Proposition 2}\label{app:prop2}
\begin{proof}
We proof \cshref{prop:tol_reg_split} as a corollary of \cshref{prop:split_mc_tol}. Note that $\widehat{P}(\epsilon,\delta)$ equals $\widehat{Q}(\alpha)$ if and only if
\begin{equation}
    \frac{1}{n+1}\big(\sup\big\{k: \mathrm{Bin}(k; n, \epsilon)\leq\delta\big\}+1\big)\leq\alpha< \frac{1}{n+1}\big(\sup\big\{k: \mathrm{Bin}(k; n, \epsilon)\leq\delta\big\}+2\big),
\end{equation}
where the strict inequality follows from the ceiling function in $\widehat{Q}(\alpha)$. For all $\alpha$ in this interval, substitution into \cref{eq:delta_cond} yields
\begin{equation}
    \delta\geq\mathrm{Bin}(\sup\big\{k: \mathrm{Bin}(k; n, \epsilon)\leq\delta\big\}; n,\epsilon)
\end{equation}
which holds arbitrarily for all $\epsilon,\delta,n$. Then by \cshref{prop:split_mc_tol}, the split conformal prediction set $\mathcal{S}_{\widehat{\lambda}_\mathrm{split}}(X_{n+1})$ (\cref{eq:split_set}) with $\widehat{\lambda}_\mathrm{split}=\widehat{P}(\epsilon,\delta)$ (\cref{eq:P_hat}) is an $(\epsilon,\delta)$-tolerance region. 

If the non-conformity scores are almost surely distinct, then \cref{eq:delta_cond} is a necessary and sufficient condition for $\mathcal{S}_{\widehat{\lambda}_\mathrm{split}}(X_{n+1})$ with $\widehat{\lambda}_\mathrm{split}=\widehat{Q}(\alpha)$ to be an $(\epsilon,\delta)$-tolerance region. As \cref{eq:delta_cond} arbitrarily holds for all $\epsilon,\delta,n$ for $\widehat{\lambda}_\mathrm{split}=\widehat{P}(\epsilon,\delta)$, we obtain the last part of \cshref{prop:tol_reg_split}.

\end{proof}

\subsection{Proposition 3}\label{app:prop_new}
\begin{proof}
First we show part \textit{(i)}. This proof is an adaptation of the proof of \cite[Proposition 2a-2b]{vovk2012conditional}. Define
\begin{equation}
    r^*:=\sup\big\{r: \mathbb{P}\big[r(X_{n+1},Y_{n+1})> r \ \big| \ \{(X_i,Y_i)\}^n_{i=1}\big] > \epsilon\big\}.
\end{equation}
as the largest value of the score function such that the probability of observing a score larger than $r^*$ given the calibration set is larger than $\epsilon$. Further define
\begin{equation}
\begin{split}
    \epsilon'&:=\mathbb{P}\big[r(X_{n+1},Y_{n+1})> r^* \ \big| \ \{(X_i,Y_i)\}^n_{i=1}\big],\\
    \epsilon''&:=\mathbb{P}\big[r(X_{n+1},Y_{n+1})\geq r^* \ \big| \ \{(X_i,Y_i)\}^n_{i=1}\big].
\end{split}
\end{equation}
We must have $\epsilon'\leq\epsilon\leq\epsilon''$ with equality if and only if the random score function $r(X,Y)$ (\cref{eq:score}) follows a continuous distribution. 
    
For both $\epsilon'=\epsilon$ and $\epsilon'<\epsilon$, the probability of miscoverage conditional on the calibration set is greater than $\epsilon$ if and only if $\widehat{Q}(\alpha)< r^*$. Recall that $\widehat{Q}(\alpha)=r_{(\lceil (1-\alpha)(n+1)\rceil)}$, i.e. the $\lceil (1-\alpha)(n+1)\rceil$-th largest element of $\{r_i\}^n_{i=1}$. Then $\widehat{Q}(\alpha)< r^*$ if and only if at most $\lfloor \alpha(n+1)-1\rfloor$ points in the calibration set $\{(X_i,Y_i)\}^{n}_{i=1}$ are such that $r(X_i,Y_i)\geq r^*$. Since the sample is iid, the probability of at most $\lfloor \alpha(n+1)-1\rfloor$ calibration points such that $r(X_i,Y_i)\geq r^*$ equals
\begin{equation}\label{eq:bin_ineq}
    \mathbb{P}\Big[B''\leq\lfloor \alpha(n+1)-1\rfloor  \Big] \leq \mathbb{P}\Big[B\leq\lfloor \alpha(n+1)-1\rfloor  \Big],
\end{equation}
where $B''\sim\mathrm{Bin}(n,\epsilon'')$ and $B\sim\mathrm{Bin}(n,\epsilon)$ follow a Binomial distribution with $n$ trials and success probability $\epsilon''$ and $\epsilon$ respectively. The inequality follows from the property that for fixed $k$ and $m$, the Binomial cumulative distribution function $\mathrm{Bin}(k; m,p)$ is decreasing in the success probability $p$. See \cite{vovk2012conditional} (p.6) for a straightforward proof of this Binomial property.
    
We know that $\mathcal{S}_{\widehat{\lambda}_\mathrm{split}}(X_{n+1})$ is an $(\epsilon,\delta)$-tolerance region if
\begin{equation}\label{eq:tol_reg_split}
    \delta\geq\mathbb{P}\Big[\mathbb{P}\big[r(X_{n+1},Y_{n+1})> \widehat{Q}(\alpha) \ \big| \ \{(X_i,Y_i)\}^n_{i=1}\big] > \epsilon \Big].
\end{equation}
Using \cref{eq:bin_ineq}, we know that \cref{eq:tol_reg_split} holds if (\textit{not only if})
\begin{equation}\label{eq:tol_reg_split_suff}
    \delta\geq\mathbb{P}\Big[B\leq\lfloor \alpha(n+1)-1\rfloor \Big]=\mathrm{Bin}(\lfloor \alpha(n+1)-1\rfloor;n,\epsilon),
\end{equation}
which shows that the split conformal prediction set $\mathcal{S}_{\widehat{\lambda}_\mathrm{split}}(X_{n+1})$ is an $(\epsilon,\delta)$-tolerance region if condition \cref{eq:delta_cond} holds. \cref{eq:eps_cond} follow directly by inverting the Binomial distribution in \cref{eq:delta_cond} with respect to $\epsilon$. 
    
For the last part of part \textit{(i)}, assume that the non-conformity scores are almost surely distinct, i.e. that the random score function $r(X,Y)$ (\cref{eq:score}) follows a continuous distribution. Then it must hold that $\epsilon'=\epsilon=\epsilon''$ and thus that $\mathbb{P}\Big[B''\leq\lfloor \alpha(n+1)-1\rfloor  \Big] = \mathbb{P}\Big[B\leq\lfloor \alpha(n+1)-1\rfloor \Big]$ in \cref{eq:bin_ineq}. In turn, this implies that \cref{eq:tol_reg_split} holds \textit{if and only if} \cref{eq:tol_reg_split_suff} holds. 

Now we proof part \textit{(ii)}. Note that $\widehat{P}(\epsilon,\delta)$ equals $\widehat{Q}(\alpha)$ if and only if
\begin{equation}
    \frac{1}{n+1}\big(\sup\big\{k: \mathrm{Bin}(k; n, \epsilon)\leq\delta\big\}+1\big)\leq\alpha< \frac{1}{n+1}\big(\sup\big\{k: \mathrm{Bin}(k; n, \epsilon)\leq\delta\big\}+2\big),
\end{equation}
where the strict inequality follows from the ceiling function in $\widehat{Q}(\alpha)$. Due to this ceiling function, any $\alpha$ in the above interval results in the same quantile of $\{r_i\}^n_{i=1}$. To obtain the highest marginal coverage, \cshref{prop:marg_cov_split} can be applied where $\widehat{\lambda}_\mathrm{split}=\widehat{Q}(\alpha)$ and 
\begin{equation}
    \alpha= \frac{1}{n+1}\big(\sup\big\{k: \mathrm{Bin}(k; n, \epsilon)\leq\delta\big\}+1\big),
\end{equation}
i.e. the smallest value of $\alpha$ in the interval above. This results in part \textit{(ii)}.

This concludes the proof.

\end{proof}

\subsection{Proposition 4}\label{app:prop3}
\begin{proof}
To avoid handling ties, this proof only concerns univariate populations following a continuous distribution $F_Y$, i.e. we proof that for $0\leq r<s\leq n+1$,
\begin{equation}\label{eq:wilks_cov_cont}
    \mathbb{P}\big[Y_{(r)}\leq Y_{n+1}\leq Y_{(s)} \ \big|\  \{Y_{i}\}^n_{i=1}\big] \sim \mathrm{Beta}(s-r, n-s+r+1).
\end{equation}
For the discontinuous case, we refer to \cite[p. 192]{scheffe1945non}. In particular, they proof that 
\begin{equation}
    \mathbb{P}\big[Y_{(r)}\leq Y_{n+1}\leq Y_{(s)} \ \big|\  \{Y_{i}\}^n_{i=1}\big] \geq Z \geq \mathbb{P}\big[Y_{(r)}< Y_{n+1}< Y_{(s)} \ \big|\  \{Y_{i}\}^n_{i=1}\big],
\end{equation}
which clearly reduces to \cref{eq:wilks_cov_cont} if the $Y_i$ are almost surely distinct. The proof below is adapted from \cite[Chapter 8]{krishnamoorthy2009statistical}.
    
By a probability integral transform, if $\{Y_{i}\}^n_{i=1}$ is an iid sample from a univariate population following a continuous distribution $F_Y$, then $\{U_i\}^n_{i=1}$ with $U_i:=F_Y(Y_i)$ is a sample from a standard uniform distribution $\mathcal{U}[0,1]$. Furthermore, if $\{Y_{(i)}\}^n_{i=1}$ is the set of order statistics corresponding to the sample $\{Y_{i}\}^n_{i=1}$, then $\{U_{(i)}\}^n_{i=1}$ with $U_{(i)}:=F_Y(Y_{(i)})$ is the set of order statistics corresponding to the sample $\{U_i\}^n_{i=1}$. Note that the $Y_{(i)}$ and $U_{(i)}$ are almost surely distinct since we are sampling from a continuous distribution.

For the $r$-th order statistic $U_{(r)}$, $1\leq r\leq n$, we can combine the well-known pdf of an order statistic with the standard uniform distribution for the sample $\{U_{(i)}\}^n_{i=1}$ to obtain $U_{(r)}\sim\mathrm{Beta}(r,n-r+1)$. For the joint pdf of two order statistics $\big(U_{(r)}, U_{(s)}\big)$, $1\leq r<s\leq n$, we have pdf
\begin{equation}\label{eq:pdf_order_u_joint}
    f_{U_{(r)},U_{(s)}}(x,y)=cx^{r-1}(y-x)^{s-r-1}(1-y)^{n-s}, \quad 0<x<y<1,
\end{equation}
where $c=\frac{\Gamma(n+1)}{\Gamma(r)\Gamma(s-r)\Gamma(n-s+1)}$ is a normalizing constant. We refer to \cite{krishnamoorthy2009statistical} (p. 210-211) for a straightforward proof of this joint pdf.
    
Now consider the prediction interval $Y_{(r)}\leq Y_{n+1}\leq Y_{(s)}$ for $0\leq r<s\leq n+1$. We define $Y_{(0)}:=-\infty$ and $Y_{(n+1)}:=+\infty$ such that $U_{(0)}:=F_Y(Y_{(0)})=0$ and $U_{(n+1)}:=F_Y(Y_{(n+1)})=1$. The conditional probability that a new label $Y_{n+1}\sim F_Y$ is contained in $\big[Y_{(r)},Y_{(s)}\big]$ equals
\begin{equation}
\begin{split}
    \mathbb{P}\big[Y_{(r)}\leq Y_{n+1}\leq Y_{(s)} \ \big| \ \{Y_{i}\}^n_{i=1}\big]&=F_Y(Y_{(s)})-F_Y(Y_{(r)})\\
        &=U_{(s)}-U_{(r)}
\end{split}
\end{equation}
We recognize four cases. Fist, for $r=0$ and $s=n+1$, $Y_{n+1}$ is contained in $\big[Y_{(r)},Y_{(s)}\big]$ almost surely, but we ignore this non-informative boundary case. 
    
Second, for $r=0$ and $s<n+1$, the pdf of $U_{(s)}$ yields
\begin{equation}
    \mathbb{P}\big[-\infty<Y_{n+1}\leq Y_{(s)} \ \big|\ \{Y_{i}\}^n_{i=1}\big]=U_{(s)}\sim\mathrm{Beta}(s,n-s+1).
\end{equation}
Third, for $r>0$ and $s=n+1$ we similarly write
\begin{equation}
    \mathbb{P}\big[Y_{(r)}\leq Y_{n+1}<+\infty \ \big| \ \{Y_{i}\}^n_{i=1}\big]=1-U_{(r)}\sim\mathrm{Beta}(n-r+1,r).
\end{equation}
Finally, for $0<r<s<n+1$, we retrieve the pdf of $U_{(s)}-U_{(r)}$ through the joint pdf of $\big(U_{(r)}, U_{(s)}\big)$. Write $V=U_{(s)}-U_{(r)}$ and $W=U_{(s)}$ with inverse transformation $U_{(r)}=W-V$ and $U_{(s)}=W$ for $0<V<W<1$. Note that the Jacobian of the transformation is one. The joint pdf of $\big(V,W\big)$ is
\begin{equation}
    f_{V,W}(v,w)=c(w-v)^{r-1}v^{s-r-1}(1-w)^{n-s}, \quad 0<v<w<1.
\end{equation}
To integrate out $W$ consider the transformation $w=v+t(1-v)$ for $0<t<1$. This yields $w-v=t(1-v)$, $1-w=(1-v)(1-t)$ and $\mathrm{d}w=(1-v)\mathrm{d}t$. Then the marginal pdf of $V$ becomes
\begin{equation}
    \begin{split}
        f_V(v)&=cv^{s-r-1}\int^1_0(1-v)^{r-1}t^{r-1}(1-v)^{n-s}t^{n-s}(1-v)\mathrm{d}t\\
        &=\frac{\Gamma(n+1)}{\Gamma(s-r)\Gamma(n-s+r+1)}v^{s-r-1}(1-v)^{n-s+r}, \quad 0<v<1,
    \end{split}
\end{equation}
by recognizing a $\mathrm{Beta}(r, n-s+1)$ distribution and substituting $c=\frac{\Gamma(n+1)}{\Gamma(r)\Gamma(s-r)\Gamma(n-s+1)}$. Clearly, $V=U_{(s)}-U_{(r)}\sim\mathrm{Beta}(s-r, n-s+r+1)$ conditional on the sample $\{Y_{i}\}^n_{i=1}$. This concludes the proof. 

\end{proof}

\subsection{Proposition 5}\label{app:prop4}
\begin{proof}
If the sample $\{(X_i,Y_i)\}^{n+1}_{i=1}$ is iid, then the non-conformity scores $\{r_i\}^n_{i=1}$ are iid as well, conditional on the proper training set. Consider the interval $\big[r_{(0)},r_{(\lceil (1-\alpha)(n+1)\rceil)}\big]=\big(-\infty,r_{(\lceil (1-\alpha)(n+1)\rceil)}\big]$. Then by \cshref{prop:wilks_cov} we have
\begin{equation}\label{eq:dist_cov_score_disc}
    \mathbb{P}\big[r(X_{n+1},Y_{n+1})\in \big(-\infty,r_{(\lceil (1-\alpha)(n+1)\rceil)}\big] \ \big| \ \{r_i\}^n_{i=1}\big]\geq Z,
\end{equation}
where $Z \sim \mathrm{Beta}(\lceil (1-\alpha)(n+1)\rceil,\lfloor \alpha(n+1)\rfloor)$, with equality if the $r_i$ are almost surely distinct.
    
For any $(X_{n+1}, Y_{n+1})\in\mathcal{X}\times\mathcal{Y}$, \cref{eq:split_set_score} shows that $Y_{n+1}\in \mathcal{S}_{\widehat{\lambda}_\mathrm{split}}(X_{n+1})$ if and only if $r(X_{n+1},Y_{n+1})\in\big(-\infty,r_{(\lceil (1-\alpha)(n+1)\rceil)}\big]$. Then \cref{eq:dist_cov_exp_nonc} follows directly from \cref{eq:dist_cov_score_disc} with equality if the non-conformity scores are almost surely distinct. This concludes the proof. 

\end{proof}

\subsection{Proposition 6}\label{app:prop5}

\begin{proof}
This is an adaptation of \cite[Section 3.2]{angelopoulos2021gentle}. Conditional on the calibration set, the value of $\widehat{\lambda}_{\mathrm{split},j}=\widehat{Q}_j(\alpha)$ and thus the prediction set $\mathcal{S}_{\widehat{\lambda}_{\mathrm{split}, j}}(\cdot)$ is fixed. Then assuming the sample $\{(X_{i,j},Y_{i,j})\}_{i=1}^{n+n_\mathrm{test}}$ is iid, the empirical coverage $C_j$ is the average of $n_\mathrm{test}$ iid indicator random variables and thus $C_j\sim\frac{1}{n_\mathrm{test}}\mathrm{Bin}(n_\mathrm{test},\mu)$. The conditional mean $\mu$ of $C_j$ is
\begin{equation}
    \mu=\mathbb{E}\big[C_j\ \big| \ \{(X_{i,j},Y_{i,j})\}^n_{i=1}\big]=\mathbb{P}\big[Y_{n+1}\in \mathcal{S}_{\widehat{\lambda}_\mathrm{split}}(X_{n+1}) \ \big| \ \{(X_i,Y_i)\}^n_{i=1}\big],
\end{equation}
which by \cshref{prop:dist_exp_cov} we know stochastically dominates the random variable $Z$, with equality $\mu=Z$ if the non-conformity scores are almost surely distinct. This concludes the proof.

\end{proof}

\section{Computational Tables}\label{sec:tables}

\begin{table}[h!]
\setlength\tabcolsep{4pt}
\centering
\begin{tabular}{ccccccccccc}
\toprule
\multicolumn{1}{c}{} & \multicolumn{5}{c}{\textbf{$n=100$}} & \multicolumn{5}{c}{$n=1000$} \\
\cmidrule(rl){2-6} \cmidrule(rl){7-11} 
$\delta$ & $\epsilon: 10\%$ & $5\%$ & $1\%$ & $0.5\%$ & $0.1\%$ & $\epsilon: 10\%$ & $5\%$ & $1\%$ & $0.5\%$ & $0.1\%$  \\
\midrule
\rowcolor{lavender}
$10\%$ & 5 & 1 & 0 & 0 & 0 & 87 & 40 & 5 & 1 & 0  \\
$5\%$ & 4 & 1 & 0 & 0 & 0 & 84 & 38 & 4 & 1 & 0\\
\rowcolor{lavender}
$1\%$ & 3 & 0 & 0 & 0 & 0 & 78 & 34 & 2 & 0 & 0  \\
$0.5\%$ & 2 & 0 & 0 & 0 & 0 & 75 & 32 & 2 & 0 & 0  \\
\rowcolor{lavender}
$0.1\%$ & 1 & 0 & 0 & 0 & 0 & 71 & 29 & 1 & 0 & 0 \\
\bottomrule
\end{tabular}

\begin{tabular}{ccccccccccc}
\\
\toprule
\multicolumn{1}{c}{} & \multicolumn{5}{c}{\textbf{$n=10000$}} & \multicolumn{5}{c}{$n=100000$} \\
\cmidrule(rl){2-6} \cmidrule(rl){7-11} 
$\delta$ & $\epsilon: 10\%$ & $5\%$ & $1\%$ & $0.5\%$ & $0.1\%$ & $\epsilon: 10\%$ & $5\%$ & $1\%$ & $0.5\%$ & $0.1\%$  \\
\midrule
\rowcolor{lavender}
$10\%$ &  961 & 471 & 86 & 40 & 5 & 9878 & 4911 & 959 & 471 & 86   \\
$5\%$ & 950 & 463 & 83 & 38 & 4 & 9843 & 4886 & 948 & 463 & 83 \\
\rowcolor{lavender}
$1\%$ & 930 & 449 & 77 & 33 & 2 & 9779 & 4839 & 927 & 448 & 77  \\
$0.5\%$ & 922 & 444 & 74 & 32 & 2 & 9755 & 4822 & 919 & 442 & 74  \\
\rowcolor{lavender}
$0.1\%$ & 907 & 433 & 70 & 29 & 1 & 9707 & 4787 & 903 & 432 & 70 \\
\bottomrule
\end{tabular}

\caption[Computational table for $\sup\big\{k: \mathrm{Bin}(k; n, \epsilon)\leq\delta\big\}$.]{Solution to $\sup\big\{k: \mathrm{Bin}(k; n, \epsilon)\leq\delta\big\}$ for pre-specified $\epsilon, \delta\in(0,1)$ and a calibration set of size $n$.}
\end{table}

\begin{table}[h!]
\setlength\tabcolsep{2.5pt}
\centering
\begin{tabular}{ccccccccccc}
\toprule
\multicolumn{1}{c}{\%} & \multicolumn{5}{c}{\textbf{$n=100$}} & \multicolumn{5}{c}{$n=1000$} \\
\cmidrule(rl){2-6} \cmidrule(rl){7-11} 
$\delta$ & $\alpha: 10\%$ & $5\%$ & $1\%$ & $0.5\%$ & $0.1\%$ & $\alpha:10\%$ & $5\%$ & $1\%$ & $0.5\%$ & $0.1\%$ \\
\midrule
\rowcolor{lavender}
$10\%$ & 13.8351 & 7.8347 & 2.2762 & 0.0000 & 0.0000 & 11.2203 & 5.8942 & 1.4169 & 0.7977 & 0.2299  \\
$5\%$ & 15.1795 & 8.9196 & 2.9513 & 0.0000 & 0.0000 & 11.5924 & 6.1758 & 1.5652 & 0.9129 & 0.2991  \\
\rowcolor{lavender}
$1\%$ & 17.8746 & 11.1704 & 4.5007 & 0.0000 & 0.0000 & 12.3092 & 6.7257 & 1.8691 & 1.1560 & 0.4594  \\
$0.5\%$ & 18.9152 & 12.0632 & 5.1604 & 0.0000 & 0.0000 & 12.5776 & 6.9342 & 1.9888 & 1.2540 & 0.5284  \\
\rowcolor{lavender}
$0.1\%$ & 21.1465 & 14.0165 & 6.6745 & 0.0000 & 0.0000 & 13.1413 & 7.3760 & 2.2503 & 1.4714 & 0.6883 \\
\bottomrule
\end{tabular}

\begin{tabular}{ccccccccccc}
\\
\toprule
\multicolumn{1}{c}{\%} & \multicolumn{5}{c}{\textbf{$n=10000$}} & \multicolumn{5}{c}{$n=100000$} \\
\cmidrule(rl){2-6} \cmidrule(rl){7-11} 
$\delta$ & $\alpha: 10\%$ & $5\%$ & $1\%$ & $0.5\%$ & $0.1\%$ & $\alpha:10\%$ & $5\%$ & $1\%$ & $0.5\%$ & $0.1\%$ \\
\midrule
\rowcolor{lavender}
$10\%$ & 10.3850 & 5.2806 & 1.1293 & 0.5921 & 0.1420 & 10.1216 & 5.0884 & 1.0405 & 0.5287 & 0.1130  \\
$5\%$ & 10.4968 & 5.3629 & 1.1689 & 0.6213 & 0.1569 & 10.1563 & 5.1138 & 1.0522 & 0.5372 & 0.1169  \\
\rowcolor{lavender}
$1\%$ & 10.7085 & 5.5196 & 1.2456 & 0.6783 & 0.1877 & 10.2217 & 5.1616 & 1.0746 & 0.5533 & 0.1247  \\
$0.5\%$ & 10.7865 & 5.5776 & 1.2744 & 0.7001 & 0.1998 & 10.2457 & 5.1791 & 1.0828 & 0.5593 & 0.1276  \\
\rowcolor{lavender}
$0.1\%$ & 10.9486 & 5.6985 & 1.3353 & 0.7462 & 0.2264 & 10.2953 & 5.2154 & 1.1000 & 0.5717 & 0.1337 \\
\bottomrule
\end{tabular}

\caption[Computational table for $\epsilon \geq \inf\big\{p: \mathrm{Bin}(\lfloor \alpha(n+1)-1\rfloor; n, p)\leq \delta\big\}$.]{Smallest significance level $\epsilon\in(0,1)$ (\%), rounded up to 4 digits, that satisfies $\epsilon \geq \inf\big\{p: \mathrm{Bin}(\lfloor \alpha(n+1)-1\rfloor; n, p)\leq \delta\big\}$
for pre-specified $\alpha,\delta\in(0,1)$ and a calibration set of size $n$.}
\end{table}


\end{document}